\documentclass{article}

\usepackage{microtype}
\usepackage{graphicx}
\usepackage{subcaption}
\usepackage{booktabs} 
\usepackage{enumitem}
\usepackage{hyperref}
\usepackage[table]{xcolor}
\usepackage{xcolor}
\usepackage{tcolorbox}
\usepackage{amsmath}
\tcbuselibrary{breakable,skins}

\definecolor{dsdrgreen}{RGB}{220,245,230}  

\usepackage{bbm}
\usepackage[preprint]{icml2026}

\usepackage{amsmath}
\usepackage{amssymb}
\usepackage{mathtools}
\usepackage{amsthm}

\usepackage[capitalize,noabbrev]{cleveref}

\theoremstyle{plain}
\newtheorem{theorem}{Theorem}[section]
\newtheorem{proposition}[theorem]{Proposition}
\newtheorem{lemma}[theorem]{Lemma}

\theoremstyle{definition}

\theoremstyle{remark}

\usepackage[textsize=tiny]{todonotes}

\icmltitlerunning{DSDR: Dual-Scale Diversity Regularization for Exploration in LLM Reasoning}

\begin{document}

\twocolumn[
  \icmltitle{\textit{DSDR}: Dual-Scale Diversity Regularization for Exploration in LLM Reasoning}

  \icmlsetsymbol{equal}{*}
  \icmlsetsymbol{leader}{$\dagger$}

  \begin{icmlauthorlist}
    \icmlauthor{Zhongwei Wan}{equal,osu}
    \icmlauthor{Yun Shen}{equal,osu}
    \icmlauthor{Zhihao Dou}{case}
    \icmlauthor{Donghao Zhou}{cuhk}
    \icmlauthor{Yu Zhang}{mq}
    \icmlauthor{Xin Wang}{osu}
    \icmlauthor{Hui Shen}{umich}
    \icmlauthor{Jing Xiong}{hku}
    \icmlauthor{Chaofan Tao}{hku}
    \icmlauthor{Zixuan Zhong}{ucl}
    \icmlauthor{Peizhou Huang}{umich}
    \icmlauthor{Mi Zhang}{osu}
  \end{icmlauthorlist}

  \icmlaffiliation{osu}{The Ohio State University, USA}
  \icmlaffiliation{case}{Case Western Reserve University, USA}
  \icmlaffiliation{cuhk}{The Chinese University of Hong Kong, Hong Kong}
  \icmlaffiliation{mq}{Macquarie University, Australia}
  \icmlaffiliation{umich}{University of Michigan, USA}
  \icmlaffiliation{hku}{The University of Hong Kong, Hong Kong}
  \icmlaffiliation{ucl}{University College London, UK}

  \icmlcorrespondingauthor{Zhongwei Wan}{wan.512@osu.edu}
  \icmlcorrespondingauthor{Mi Zhang}{mizhang.1@osu.edu}

  \icmlkeywords{Reinforcement Learning, Large Language Models, Exploration, Diversity Regularization}

  \vskip 0.3in
]

\printAffiliationsAndNotice{}

\begin{abstract}
Reinforcement learning with verifiers (RLVR) is a central paradigm for improving large language model (LLM) reasoning, yet existing methods often suffer from limited exploration. Policies tend to collapse onto a few reasoning patterns and prematurely stop deep exploration, while conventional entropy regularization introduces only local stochasticity and fails to induce meaningful path-level diversity, leading to weak and unstable learning signals in group-based policy optimization.
We propose \textbf{DSDR}, a \textit{\textbf{D}ual-\textbf{S}cale \textbf{D}iversity \textbf{R}egularization} reinforcement learning framework that decomposes diversity in LLM reasoning into global and coupling components. Globally, DSDR promotes diversity among correct reasoning trajectories to explore distinct solution modes. Locally, it applies a length-invariant, token-level entropy regularization restricted to correct trajectories, preventing entropy collapse within each mode while preserving correctness. The two scales are coupled through a global-to-local allocation mechanism that emphasizes local regularization for more distinctive correct trajectories.
%
We provide theoretical support showing that DSDR preserves optimal correctness under bounded regularization, sustains informative learning signals in group-based optimization, and yields a principled global-to-local coupling rule.
Experiments on multiple reasoning benchmarks demonstrate consistent improvements in accuracy and pass@k, highlighting the importance of dual-scale diversity for deep exploration in RLVR. Code is available at \href{https://github.com/SUSTechBruce/DSDR}{DSDR}.
\end{abstract}

\section{Introduction}

Reinforcement learning with verified reward (RLVR)~\citep{liu2024deepseek, shao2024deepseekmath} has recently emerged as a powerful paradigm for enhancing the reasoning capabilities of LLMs. Group-based policy optimization methods, such as GRPO~\citep{guo2025deepseek}, further improve training stability by exploiting relative comparisons among sampled solutions, making RLVR practical at scale.
By leveraging outcome-based supervision rather than token-level imitation, RLVR has enabled substantial improvements in math and code reasoning, logical inference, and multi-step problem solving, and has become a core component of recent advances in reasoning-oriented LLM training~\citep{comanici2025gemini, singh2025openai, yang2025qwen3}.

\begin{figure*}[t]
  \centering
  \begin{minipage}{0.58\textwidth}
    \centering
    \includegraphics[width=\linewidth]{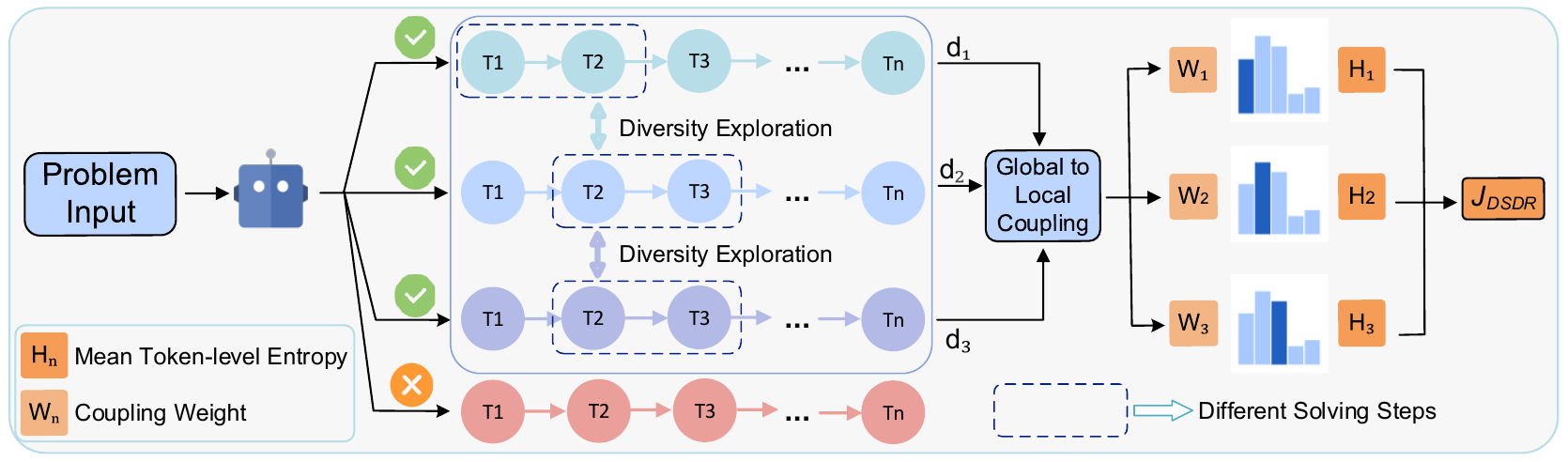}
  \end{minipage}
  \hfill
  \begin{minipage}{0.2\textwidth}
    \centering
    \includegraphics[width=\linewidth]{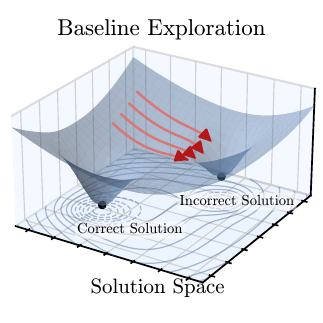}
  \end{minipage}
  \hfill
  \begin{minipage}{0.2\textwidth}
    \centering
    \includegraphics[width=\linewidth]{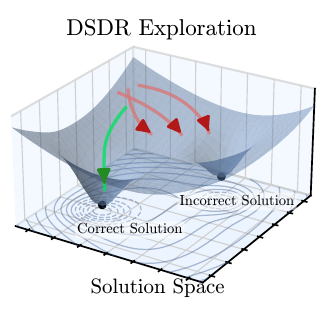}
  \end{minipage}

  \caption{
  (\textbf{Left}): global-to-local coupling for enhanced exploration during RL training.
  (\textbf{Right}): baseline exploration collapses to local suboptimal solutions, while DSDR promotes diverse trajectories that escape local optima and reach the correct solution space.
  }
  \vskip -0.6em
  \label{fig:motivation}
\end{figure*}






%
Despite these successes, verified reward-maximizing RL training often exhibits limited deep exploration, even when alternative valid solution paths exist~\citep{liu2025prorl, shen2025entropy, wu2025invisible, chen2025eepoexplorationenhancedpolicyoptimization}. This phenomenon is widely observed across RLVR pipelines: while models improve pass@1 accuracy, they often do so by concentrating probability mass on a small set of homogeneous reasoning patterns, leading to a collapse in solution diversity. Consequently, pass@k performance fails to improve and generalization deteriorates, especially when models are evaluated on out-of-domain or more compositional reasoning tasks~\cite{walder2025pass, jiang2025risk}.

A natural response is to encourage diversified exploration during training, yet existing methods remain insufficient. Entropy regularization~\cite{shen2025entropy, chen2025pass, agarwal2025unreasonable}, widely used in RL and RLVR, injects token-level stochasticity but mainly induces local randomness and fails to promote distinct reasoning paths.
Conversely, recent diversity-driven methods~\cite{zhang2025right, chen2025post, li2025jointlyreinforcingdiversityquality, hu2025diversity} encourage variation among generated solutions, sometimes with quality or correctness constraints. However, these approaches are typically single-scale or weakly coupled across scales: token-level entropy or uncertainty control induces local stochasticity and rarely sustains distinct reasoning trajectories, while trajectory-level diversity alone fails to prevent intra-mode entropy collapse once a few correct templates dominate. Consequently, policies may still prematurely concentrate on a small set of correct reasoning modes, and in group-normalized RLVR this concentration weakens within-group preference signals as verifier rewards become nearly constant. The core tension between deep exploration and correctness therefore remains unresolved, requiring exploration to be correctness-aligned and explicitly coordinated across both trajectory- and token-scales.


%
This motivates a dual-scale formulation of exploration for LLM reasoning. (i) At the global level, exploration requires discovering and maintaining multiple distinct reasoning modes, corresponding to different solution paths. (ii) At the local level, exploration requires preventing premature entropy collapse~\cite{cui2025entropy, shen2025entropy} within each mode, so that correct trajectories remain robust and expressive rather than brittle or over-confident. Crucially, these two forms of diversity are complementary and can be jointly optimized rather than treated in isolation, since not all correct modes are equally valuable for further exploration, motivating a mechanism that allocates local regularization based on global distinctiveness.

Building on this insight, we propose \textbf{DSDR}, a \textit{\textbf{D}ual-\textbf{S}cale \textbf{D}iversity \textbf{R}egularization} framework for RL-based LLM reasoning. DSDR integrates global diversity regularization over correct reasoning trajectories with a length-invariant, token-level entropy term applied exclusively to correct solutions. As illustrated in Figure~\ref{fig:motivation}, these two scales are coupled through a global-to-local allocation mechanism, which prioritizes local regularization for more distinctive correct trajectories. Additionally, this coupling prevents exploration from collapsing into narrow, locally suboptimal reasoning templates, instead promoting diverse trajectories that escape local optima and expand the correct solution space.
We further provide theoretical support for DSDR, showing that bounded positive-only local entropy preserves optimal correctness, while correct-only global shaping prevents signal degeneracy in group-normalized optimization. We also justify the global-to-local softmax coupling from a principled objective view, explaining why dual-scale diversity strengthens learning signals.
Extensive experiments across multiple reasoning benchmarks demonstrate that DSDR consistently improves accuracy, pass@k performance, and training stability, highlighting the importance of principled dual-scale diversity for deep exploration in RLVR. 
Our main contributions are summarized as follows:
\begin{itemize}[leftmargin=*]
    \item We introduce a dual-scale perspective on exploration in LLM reasoning, explicitly distinguishing global (inter-mode) and local (intra-mode) diversity and clarifying their complementary roles in RLVR.
    \item We propose DSDR, a correctness-aligned dual-scale diversity regularization framework that couples global diversity with positive-only, length-invariant local entropy through a global-to-local allocation mechanism.
    \item We provide theoretical support for correctness preservation and signal preservation in group-normalized RLVR, together with a principled interpretation of the global-to-local coupling, validated by consistent empirical gains.
\end{itemize}


\section{Related Work}

\begin{figure*}[t]
    \centering
    \includegraphics[width=0.99\linewidth]{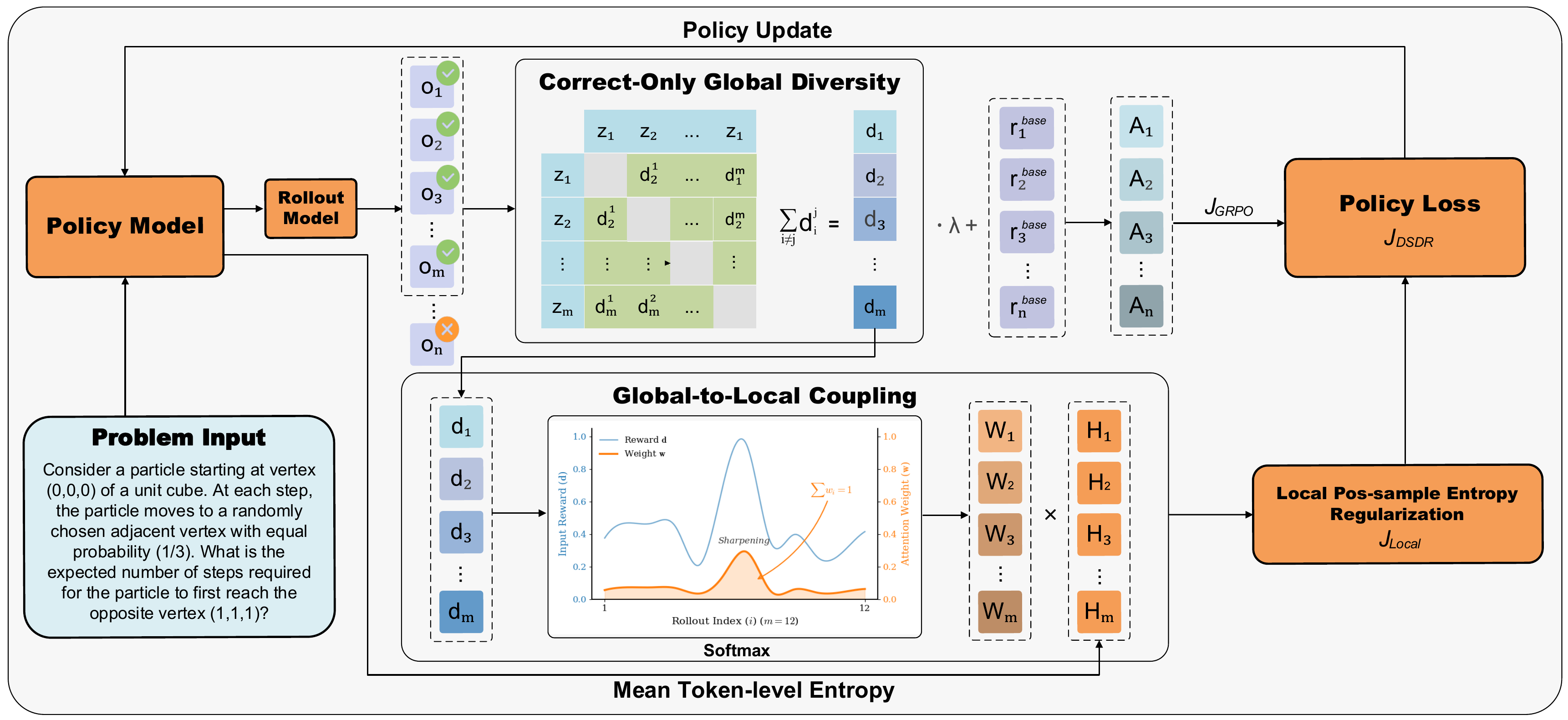}
    \caption{DSDR training pipeline for dual-scale exploration in RL.
Correct-only global diversity promotes exploration across solution modes,
while a global-to-local coupling mechanism allocates length-invariant local entropy regularization to distinctive correct trajectories.
Both signals are integrated into policy updates to enable deep exploration without sacrificing correctness.}
    \label{fig:dsdr_pipeline}
\end{figure*}

\noindent \textbf{RLVR and Exploration in LLMs.} 
Reinforcement learning with verifiable rewards (RLVR) has become a prominent approach for improving LLM reasoning~\cite{cobbe2021training, singh2025openai, guo2025deepseek}. While this training can elicit emergent reasoning behaviors such as verification and self-reflection~\cite{gandhi2025cognitive, wan2025srpo}, it often suffers from limited exploration, where policies converge early to a narrow set of reasoning patterns, resulting in performance plateaus. To mitigate this issue, prior work has explored a range of exploration-enhancing strategies, including increasing policy stochasticity through entropy regularization or temperature adjustment~\cite{hou2025advancing}, modifying optimization objectives via relaxed clipping or pass@k-based rewards~\cite{yu2025dapo, chen2025pass}, and intervening in rollout dynamics to encourage mode switching, such as the sample-then-forget mechanism~\cite{chen2025eepoexplorationenhancedpolicyoptimization}. While these approaches improve exploration from different angles, they either rely on unstructured randomness, objective-level relaxation, or rollout-level interventions, and do not explicitly model how exploration should be coordinated across different scales of reasoning.
Instead, our work attempt to address exploration limitations by structuring diversity directly within policy optimization across global-to-local scales, enabling deep exploration without modifying rollout procedures.

\noindent \textbf{Diversity and Entropy Control for LLM Reasoning. }
Recent studies have explored promoting diversity in LLM reasoning by manipulating uncertainty at different levels of the policy. Token-level methods selectively encourage stochastic actions through entropy bonuses, clipping, or KL constraints~\cite{cui2025entropy, liu2025ettrl, yu2025dapo, agarwal2025unreasonable, shen2025entropy, yao2025diversity}, which can alleviate premature collapse but primarily operate at a local action level. While effective in increasing short-term randomness, these methods do not explicitly encourage diversity across complete reasoning trajectories.
More recent approaches consider diversity at a global level. \citet{chen2025pass} and \citet{walder2025pass} leverage pass@k as a training signal to encourage multiple candidate solutions, while, in a concurrent work, \citet{cui2025entropy} train a partitioning classifier to measure and amplify diversity in the advantage function. Closely related, some approaches~\cite{chen2025post, li2025jointlyreinforcingdiversityquality, hu2025diversity} promotes global diversity among candidate solutions to improve deep exploration. However, above methods treat global and local diversity signals largely independently and do not specify how diversity at different scales should interact during optimization. In contrast, DSDR explicitly decomposes diversity into global and local components and couples them through a global-to-local allocation mechanism, which adaptively concentrates local entropy regularization on more distinctive correct reasoning trajectories. 


\section{Methodology}
\label{sec:method}

\subsection{Preliminaries}
\label{sec:prelim}

We briefly review Group Relative Policy Optimization (GRPO)~\cite{guo2025deepseek}, which serves as the
optimization backbone for reinforcement learning with verifiable rewards (RLVR)
in reasoning tasks.
Given an input prompt $q$, a policy $\pi_\theta$ generates an output sequence
$o=(o_1,\ldots,o_T)$ following autoregressive factorization
\vskip -0.4em
\begin{equation}
\pi_\theta(o \mid q)
=
\prod_{t=1}^{T} \pi_\theta(o_t \mid q, o_{<t}).
\end{equation}
\vskip -0.4em
A verifier provides a scalar reward $r = R(q,o)$ for the completed sequence, which
is typically binary for verifiable reasoning tasks.
GRPO samples a group of $G$ candidate outputs $\{o_i\}_{i=1}^{G}$ from a lagged
behavior policy $\pi_{\theta_{\text{old}}}$ and computes rewards
$\{r_i\}_{i=1}^{G}$.
To obtain a group-relative learning signal, rewards are normalized within each
group to form advantages.
Specifically, the advantage $A_i$ for each response is computed as
\vskip -0.4em
\begin{equation}
\label{eq:advantage}
A_i
=
\frac{r_i - \mathrm{mean}(r_1,r_2,\ldots,r_G)}
{\mathrm{std}(r_1,r_2,\ldots,r_G)},
\end{equation}
\vskip -0.4em
where $\{r_i\}_{i=1}^{G}$ are rewards from the group and
$\mathrm{std}(\cdot)$ denotes the standard deviation with a small constant added
for numerical stability.
Optimization is performed using a PPO-style clipped surrogate objective at the
token level. Let $T_i$ denote the length of $o_i$, and define the token-wise
importance ratio
$\rho_{i,t} = \pi_\theta(o_{i,t}\mid q,o_{i,<t}) /
\pi_{\theta_{\text{old}}}(o_{i,t}\mid q,o_{i,<t})$.
The GRPO objective is then given by
\begin{equation}
\label{eq:grpo}
\begin{aligned}
J_{\textsc{grpo}}(\theta)
&=\mathbb{E}_{q}\Big[
\tfrac{1}{G}\sum_{i=1}^{G}\tfrac{1}{T_i}\sum_{t=1}^{T_i}
\min\!\big(\rho_{i,t}A_i,
\\
&\hspace{-4.1em}
\operatorname{clip}(\rho_{i,t},1-\epsilon_c,1+\epsilon_c)\,A_i\big)
\Big]
-\beta\, D_{\mathrm{KL}}\!\left(\pi_\theta(\cdot\!\mid q)\,\|\,\pi_{\text{ref}}(\cdot\!\mid q)\right).
\end{aligned}
\end{equation}
where $\epsilon_c$ denotes the clipping threshold, $\beta$ controls the KL
regularization strength, and $\pi_{\text{ref}}$ is a fixed reference policy.
GRPO leverages relative comparisons within each group to stabilize optimization, but its learning signal critically depends on reward variation across sampled trajectories.

\subsection{DSDR: Dual-Scale Diversity Regularization}
\label{sec:dsdr}
We adopt the group-based RLVR training protocol defined earlier: for each prompt $q$, we sample a group of $G$ rollouts $\{o_i\}_{i=1}^{G} \sim \pi_{\theta_{\mathrm{old}}}(\cdot \mid q)$ and obtain verifiable rewards $r_i \in \{0,1\}$. DSDR augments the backbone with two diversity regularizers that operate at different scales and are explicitly coupled, as shown in Figure~\ref{fig:dsdr_pipeline}. At the \textbf{global (trajectory) scale}, DSDR assigns extra credit to \emph{correct} solutions that are more distinct within the group, which keeps the learning signal informative even when many rollouts are correct and prevents premature convergence to a single reasoning template. At the \textbf{local (token) scale}, DSDR encourages controlled entropy along positive trajectories to avoid the typical \textit{correct-mode collapse} where the model becomes highly confident token-by-token and loses nearby correct variants. The coupling is key: global distinctiveness determines where local entropy should be strongest, so local regularization expands probability mass around unique correct paths rather than uniformly perturbing all positives.

\subsubsection{Global-Scale Diversity Signals}
\label{sec:div_pattern}

For each rollout $o_i$ in a group $\{o_1, \ldots, o_G\}$, we compute a bounded
per-response diversity score $d(o_i) \in [0,1]$. The design goal is pragmatic:
(i) it should reflect \emph{trajectory-level} differences (not merely token noise),
(ii) it should be \emph{cheap} relative to rollout generation, and
(iii) it should be \emph{well-scaled} so it can be safely mixed into RLVR rewards without dominating correctness. We use two coupling signals.

\noindent \textbf{Semantic Level.}
Let $f_\phi$ be a frozen text encoder that maps a full response $o_i$ to a vector
$z_i \in \mathbb{R}^d$. We normalize embeddings so that cosine similarity becomes
a stable inner product:
\begin{equation}
z_i = f_\phi(o_i), \qquad
\bar{z}_i = \frac{z_i}{\lVert z_i \rVert_2}.
\label{eq:embedding_norm}
\end{equation}
Given two responses $o_i, o_j$, we define their embedding dissimilarity via cosine
distance. We scale it into $[0,1]$ to make it numerically comparable with other
bounded components:
\begin{equation}
\tilde{d}^{\mathrm{emb}}(o_i, o_j)
= \frac{1 - \bar{z}_i^\top \bar{z}_j}{2}
\in [0,1].
\label{eq:embedding_distance}
\end{equation}

The intuition is simple: if two responses encode similar reasoning semantics,
their embeddings align and $\tilde{d}^{\mathrm{emb}}$ is small; if they represent
different solution directions, similarity drops and the distance increases.
To turn pairwise distances into a per-response score, we use the group-average
dissimilarity:
\begin{equation}
D_{\mathrm{emb}}(o_i)
= \frac{1}{G - 1}
\sum_{j \neq i}
\tilde{d}^{\mathrm{emb}}(o_i, o_j).
\label{eq:group_avg_emb}
\end{equation}
This aggregation matters for optimization stability. A single most different
neighbor can be noisy; averaging across $G - 1$ comparisons yields a smoother
signal that is less sensitive to an outlier rollout. Computationally,
Eq.~\eqref{eq:group_avg_emb} is efficient: with
$\bar{Z} = [\bar{z}_1, \ldots, \bar{z}_G] \in \mathbb{R}^{G \times d}$, all pairwise similarities are obtained via a single matrix product $\bar{Z}\bar{Z}^\top$ followed by an elementwise transform.

\noindent \textbf{Formula Level.}
Semantic similarity alone can miss an important axis of reasoning variation in
math tasks: two solutions may appear similar at the surface level while relying
on different symbolic manipulations, or vice versa. To capture this aspect, we
introduce an Formula-level uniqueness signal, following the 
prior work~\cite{wu2024progress, chen2025post, hu2025diversity} while adopting a formulation aligned with our
group-based setting.
Let $S(o_i)$ denote the set of extracted mathematical expressions appearing in response $o_i$. For each formula $f \in S(o_i)$, we define
a binary indicator that measures whether $f$ is unique relative to the rest of
the group:
\begin{equation}
\mathbb{I}_{\mathrm{uniq}}(f, o_i)
=
\mathbbm{1}\!\left[
f \notin \bigcup_{j \neq i} S(o_j)
\right].
\label{eq:formula_indicator}
\end{equation}
The equational diversity of response $o_i$ is then computed as the average
uniqueness of its constituent formulas:
\begin{equation}
D_{\mathrm{eq}}(o_i)
=
\begin{cases}
\dfrac{1}{\lvert S(o_i) \rvert}
\sum\limits_{f \in S(o_i)} \mathbb{I}_{\mathrm{uniq}}(f, o_i),
& \lvert S(o_i) \rvert > 0, \\[6pt]
0,
& \text{otherwise}.
\end{cases}
\label{eq:equational_diversity}
\end{equation}
This definition is intentionally conservative: responses that contain no
detectable formulas contribute no equational novelty. When formulas are present,
the averaging form encourages structural diversity in symbolic reasoning while
remaining invariant to non-mathematical paraphrasing.

\noindent \textbf{Combined Global Diversity.}
Both components are bounded in $[0,1]$, so we combine them into a single global
diversity score by simple averaging:
\begin{equation}
d(o_i)
=
\frac{1}{2}
\bigl(
D_{\mathrm{emb}}(o_i)
+
D_{\mathrm{eq}}(o_i)
\bigr).
\label{eq:combined_diversity}
\end{equation}


This combination yields a bounded and well-scaled scalar signal for reward shaping across diverse reasoning tasks. The embedding-based component captures trajectory-level semantic differences broadly. while the equation-based component provides a complementary, paraphrase-robust notion of novelty when symbolic manipulations are present.

\noindent \textbf{Correct-Only Global Diversity Reward.}
We incorporate global-level diversity into RLVR only when it is consistent with the task objective. In particular, we avoid the failure case where responses are rewarded for being different despite incorrect reasoning. Accordingly, DSDR applies diversity shaping exclusively to positive rollouts and explicitly limits its influence.
To prevent reward hacking~\cite{pan2022effects} and avoid the diversity signal overpowering correctness, we apply a clipped technique~\cite{sullivan2023reward, li2023internally} only to correct rollouts and define the augmented reward as
\begin{equation}
\tilde{r}_i
=
r_i
+
\lambda_d \, \bar d_i \cdot \mathbbm{1}(r_i = 1),\;\;
\bar d_i = \operatorname{clip}\!\left(d(o_i);\,0,\,\sigma_d\right).
\label{eq:augmented_reward}
\end{equation}
Where $\lambda_d \ge 0$ controls the bonus strength and $\sigma_d$ bounds the contribution of the diversity term.
Equation~\eqref{eq:augmented_reward} addresses a concrete optimization issue in group-relative methods: when many rollouts are correct, verifier rewards can become nearly constant within a group, shrinking reward variance and weakening within-group preference gradients. By introducing controlled dispersion among correct solutions, DSDR preserves a meaningful learning signal that differentiates alternative correct trajectories without creating incentives to explore incorrect ones. A formal statement and proof are provided in Appendix~\ref{app:signal}.

\subsubsection{Global-to-Local Coupling Over Correct Trajectories }
Global diversity should not only determine which correct solutions are reinforced, but also where local entropy regularization is most effective. Intuitively, if a correct trajectory is already redundant within the group, expanding its neighborhood adds little coverage. In contrast, when a trajectory is globally distinctive, local expansion around it helps populate underexplored regions of the correct solution manifold.
Let $\mathcal{P} = \{ i \in [G] \mid r_i = 1 \}$ denote the set of correct rollouts.
We allocate local regularization strength via a diversity-weighted softmax over
correct responses:
\begin{equation}
w_i
=
\begin{cases}
\dfrac{\exp(\tau \, \bar{d}_i)}{\sum_{j \in \mathcal{P}} \exp(\tau \, \bar{d}_j)},
& i \in \mathcal{P}, \\[10pt]
0,
& i \notin \mathcal{P},
\end{cases}
\label{eq:diversity_softmax}
\end{equation}
where $\tau > 0$ is a temperature parameter. This construction defines a
probability distribution over correct rollouts (i.e., $\sum_i w_i = 1$ when
$\mathcal{P} \neq \emptyset$). As $\tau$ increases, the allocation concentrates on the most globally distinctive correct solutions; as $\tau \to 0$, it reduces to uniform weighting across correct solutions. This allocation view unifies DSDR’s dual-scale regularization: global diversity
measures inter-trajectory novelty, while local entropy concentrates exploration within trajectories where novelty is highest.
Theoretically, this diversity-softmax coupling can be derived as the
self-normalized policy-gradient weighting induced by a correct-only,
diversity-tilted objective, as claimed in Theorem~\ref{thm:dsdr_tilted_pg} and Appendix~\ref{app:proof_tilted_pg}.
A further principled optimality analysis of this softmax allocation is given in Appendix~\ref{app:allocation}.

\subsubsection{Local Positive-Sample Regularization}

A direct way to promote diversity is to encourage high entropy in the model’s
output distribution. However, for long-form generation, response-level entropy
is confounded by length: longer outputs naturally accumulate more token-level
uncertainty, so higher entropy may partially reflect more tokens. DSDR
instead adopts \emph{token-level conditional entropy}, averaged over timesteps~\cite{agarwal2025unreasonable, cui2025entropy},
so that the objective measures per-step uncertainty rather than length
accumulation.
Let $o = (o_1, \ldots, o_T) \sim \pi_{\theta_{\mathrm{old}}}(\cdot \mid q)$. We start
from the time-averaged conditional entropy:
\begin{equation}
J_{\mathrm{ent}}(\theta)
=
\mathbb{E}_{q,\; o \sim \pi_{\theta_{\mathrm{old}}}(\cdot \mid q)}
\left[
\frac{1}{T}
\sum_{t=1}^{T}
\mathcal{H}\!\left(\pi_\theta(\cdot \mid q, o_{<t})\right)
\right].
\label{eq:token_entropy_objective}
\end{equation}

Using $\mathcal{H}(\pi) = - \mathbb{E}_{a \sim \pi}[\log \pi(a)]$, each entropy term
can be written as an expectation of $\log \pi_\theta(\cdot)$. The remaining
practical issue is that rollouts are sampled from $\pi_{\theta_{\mathrm{old}}}$,
while the inner expectation is taken under $\pi_\theta$. We therefore re-express
the inner expectation using standard importance sampling~\cite{precup2000eligibility, sheng2025espo, yao2025diversity}, allowing it to be
estimated from the observed tokens $o_t$ without resampling:
\begin{equation}
\mathbb{E}_{a \sim \pi_\theta(\cdot \mid s)}[g(a)]
=
\mathbb{E}_{a \sim \pi_{\theta_{\mathrm{old}}}(\cdot \mid s)}
\left[
\frac{\pi_\theta(a \mid s)}{\pi_{\theta_{\mathrm{old}}}(a \mid s)} \, g(a)
\right].
\label{eq:importance_sampling_entropy}
\end{equation}

Applied at each timestep with $s = (q, o_{<t})$ and $g(a) = \log \pi_\theta(a \mid s)$,
this yields a tractable surrogate objective that is differentiable with respect to $\theta$ while reusing group-sampled rollouts. For each group rollout $o_i = (o_{i,1}, \ldots, o_{i,T_i})$, we define the per-token
importance ratio:
\begin{equation}
\rho_{i,t}
=
\frac{\pi_\theta(o_{i,t} \mid q, o_{i,<t})}
{\pi_{\theta_{\mathrm{old}}}(o_{i,t} \mid q, o_{i,<t})}.
\label{eq:token_ratio}
\end{equation}

DSDR then defines the local objective as
\begin{equation}
J_{\mathrm{local}}(\theta)
=
\mathbb{E}\!\left[
- \sum_{i=1}^{G}
\mathbbm{1}(r_i = 1)\, w_i\,
\frac{1}{T_i}
\sum_{t=1}^{T_i}
\rho_{i,t}\, g(o_{i,t})
\right],
\label{eq:local_objective}
\end{equation}
where $g(a) = \log \pi_\theta(a \mid q, o_{<t})$. The formulation reflects three deliberate choices. Time averaging removes incentives to increase length solely to amplify the regularizer, yielding a per-decision signal. Restricting the objective to correct rollouts ensures that local entropy refines correct trajectories rather than encouraging noise on incorrect ones. The allocation weight $w_i$ couples local and global scales, so globally distinctive solutions receive stronger local regularization, focusing
exploration on underrepresented reasoning modes.
Furthermore, an  information-theoretic decomposition and correctness-preservation guarantee under bounded local regularization are given in Appendix~\ref{app:inter_intra}.

\subsubsection{DSDR Objective}

Let $J_{\mathrm{GRPO}}(\theta; \tilde{r})$ denote the group-relative policy
optimization objective defined earlier, computed using the augmented rewards
$\tilde{r}_i$ from Eq.~\eqref{eq:augmented_reward} when forming group advantages.
DSDR optimizes
\begin{equation}
J_{\mathrm{DSDR}}(\theta)
=
J_{\mathrm{GRPO}}(\theta; \tilde{r})
+
\lambda_\ell \, J_{\mathrm{local}}(\theta),
\label{eq:dsdr_objective}
\end{equation}
where $\lambda_\ell \ge 0$ controls the strength of local regularization.
Eq.~\eqref{eq:dsdr_objective} makes the dual-scale structure explicit: the global term induces preferences among correct trajectories, while the local term mitigates token-level collapse around them, jointly enabling broad exploration with stable behavior within the correct set.

\begin{theorem}[Diversity-tilted policy gradient induces DSDR global-to-local softmax coupling]
\label{thm:dsdr_tilted_pg}
Fix a prompt $q$. Let $\bar d(q,o)\in[0,\sigma_d]$ denote a bounded (clipped) global diversity score
for a completed rollout $o$ (e.g., Eq.~\eqref{eq:augmented_reward}), and let $R(q,o)\in\{0,1\}$
be the verifiable reward.
For $\tau>0$, define the \emph{correct-only diversity-tilted} objective
\begin{equation}
\label{eq:tilted_obj}
\begin{aligned}
J_\tau(\theta;q)
&=
\frac{1}{\tau}\log Z_\tau(\theta;q),\\
Z_\tau(\theta;q)
&=
\mathbb{E}_{o\sim \pi_\theta(\cdot\mid q)}
\!\left[
\exp\!\big(\tau \bar d(q,o)\big)\,\mathbbm{1}(R(q,o)=1)
\right].
\end{aligned}
\end{equation}
Assume $Z_\tau(\theta;q)>0$. Then the policy gradient of $J_\tau$ admits the form
\begin{equation}
\label{eq:tilted_pg}
\nabla_\theta J_\tau(\theta;q)
=
\mathbb{E}_{o\sim \pi_\theta(\cdot\mid q)}
\left[
A_\tau^\theta(q,o)\;\nabla_\theta \log \pi_\theta(o\mid q)
\right],
\end{equation}
where the \emph{diversity-tilted advantage} is
\begin{equation}
\label{eq:tilted_adv}
A_\tau^\theta(q,o)
=
\frac{1}{\tau}
\left(
\frac{
\exp\!\big(\tau \bar d(q,o)\big)\,\mathbbm{1}(R(q,o)=1)
}{
Z_\tau(\theta;q)
}
-1
\right).
\end{equation}
Moreover, given i.i.d.\ rollouts $\{o_i\}_{i=1}^G\sim \pi_\theta(\cdot\mid q)$ with rewards $r_i=R(q,o_i)$,
the self-normalized Monte Carlo form of~\eqref{eq:tilted_pg} assigns weights
\begin{equation}
\label{eq:tilted_softmax_mc}
\hat w_i
=
\frac{\exp(\tau \bar d_i)\,\mathbbm{1}(r_i=1)}{\sum_{j=1}^G \exp(\tau \bar d_j)\,\mathbbm{1}(r_j=1)},
\end{equation}
which reduces to a diversity-softmax over correct rollouts:
$\hat w_i \propto \exp(\tau \bar d_i)$ on $\mathcal P=\{i:r_i=1\}$,
matching DSDR's coupling rule in Eq.~\eqref{eq:diversity_softmax}.
\end{theorem}

\section{Experiments}
\label{sec:experiments}
\vskip -0.4em
\subsection{Experiment Settings}
\label{sec:experiment settings}
\vskip -0.4em

\begin{table*}[t]
\caption{\textbf{Results on different reasoning benchmarks.}
We report \textbf{Pass@1} and \textbf{Avg@16} ($\%$) accuracy across different model scales. DSDR consistently outperforms Backbone, GRPO, and DAPO across most benchmarks and achieves the best average performance. Ablation results (w/o GD, w/o GC) show consistent performance drops, suggesting that both global diversity (GD) and global-to-local coupling (GC) regularization play important roles in DSDR.}
\label{tab:main_results}
\begin{center}
\begin{small}
\setlength{\tabcolsep}{6pt}
\begin{tabular*}{\textwidth}{@{\extracolsep{\fill}}lccccc c}
\toprule
\textbf{Models}
& \textbf{AIME24}
& \textbf{AIME25}
& \textbf{MATH500}
& \textbf{Minerva}
& \textbf{Olympiad}
& \textbf{Average} \\
\midrule

\multicolumn{7}{c}{\textbf{Qwen2.5-math-1.5B}, Max Response Length = 4K tokens} \\
\midrule
Backbone
& 10.0 / 3.0  & 3.3 / 3.0 & 27.4 / 21.0 & 5.9 / 3.9 & 13.1 / 9.9 & 11.9 / 8.2 \\
GRPO
& 16.7 / 15.2 & 3.5 / 8.5 & 57.8 / \textbf{58.3} & 14.0 / 15.0 & \textbf{24.0} / 22.9 & 23.2 / 24.0 \\
DAPO
& 13.3 / 14.4 & 6.7 / 6.3 & 54.8 / 55.3 & 15.8 / 15.2 & 22.7 / 21.4 & 22.7 / 22.5 \\
\midrule
\rowcolor{dsdrgreen}
\textbf{DSDR}
& \textbf{20.0} / \textbf{18.8} & 6.7 / \textbf{10.2} & \textbf{59.0} / 57.5 & \textbf{18.0} / 
\textbf{18.0} & 23.1 / \textbf{23.7} & \textbf{25.4} / \textbf{25.6} \\
\midrule
\midrule

\multicolumn{7}{c}{\textbf{Qwen3-1.7B}, Max Response Length = 8K tokens} \\
\midrule
Backbone
& 13.3 / 10.6 & 10.0 / 7.3 & 57.6 / 56.5 & 17.3 / 17.5 & 23.7 / 25.3 & 24.4 / 23.4 \\
GRPO
& 16.7 / 21.0 & 20.0 / 15.8 & 57.4 / 60.7 & 18.4 / 19.9 & 29.5 / 31.2 & 28.4 / 29.7 \\
DAPO
& 20.0 / 20.4 & 10.0 / 21.0 & 63.8 / 64.0 & 21.3 / 22.3 & 32.9 / 32.9  & 29.6 / 32.1 \\
\midrule
\rowcolor{dsdrgreen}
\textbf{DSDR}
& \textbf{36.7} / \textbf{32.1} & 23.2 / \textbf{27.3} & \textbf{64.4} / \textbf{65.4} & 23.3 / 23.4 & \textbf{36.4} / 36.0 & \textbf{36.8} / \textbf{36.8} \\
\hspace{2em}w/o GD
& 23.3 / 20.6 & 13.3/ 19.0 & 61.6 / 63.5 & 20.6 / 21.4 & 29.8 / 32.5 & 29.7 / 31.4 \\
\hspace{2em}w/o GC
& 30.0 / 26.5 & \textbf{23.3} / 24.0 & 63.0 / 64.6 & \textbf{23.5} / 23.4 & 35.9/ \textbf{36.1} & 35.1 / 34.9 \\

\midrule
\midrule

\multicolumn{7}{c}{\textbf{Qwen3-4B}, Max Response Length = 8K tokens} \\
\midrule

Backbone
& 13.33 / 20.63
& 16.67 / 16.04
& 64.8 / 65.75
& 26.10 / 24.61
& 32.79 / 32.69
& 30.74 / 31.94 \\

GRPO
& 36.67 / 37.50
& 33.33 / 35.21
& 64.0 / 63.83
& 25.00 / 23.62
& 37.54 / 37.70
& 39.31 / 39.57 \\

DAPO
& 33.33 / 38.38
& 40.00 / 33.13
& 59.0 / 62.74
& 25.74 / 27.37
& 36.65 / 39.42
& 38.94 / 40.21 \\

\midrule
\rowcolor{dsdrgreen}
\textbf{DSDR}
& \textbf{56.67} / \textbf{52.08}
& \textbf{50.00} / \textbf{46.46}
& \textbf{66.2} / \textbf{66.18}
& \textbf{26.84} / \textbf{27.87}
& \textbf{40.36} / \textbf{42.43}
& \textbf{48.01} / \textbf{46.80} \\

\hspace{2em}w/o GD
& 46.67 / 45.00
& 43.33 / 40.42
& 63.6 / 64.74
& 26.47 / 27.44
& 40.65 / 41.20
& 44.14 / 43.76 \\

\hspace{2em}w/o GC
& 26.67 / 44.38
& 23.33 / 35.63
& 58.0 / 62.78
& 20.21 / 24.59
& 33.98 / 40.60
& 32.04 / 41.20 \\

\bottomrule
\end{tabular*}
\end{small}
\end{center}
\end{table*}

\begin{figure*}[t]
    \vskip -0.6em
    \begin{center}
    \includegraphics[width=\textwidth]{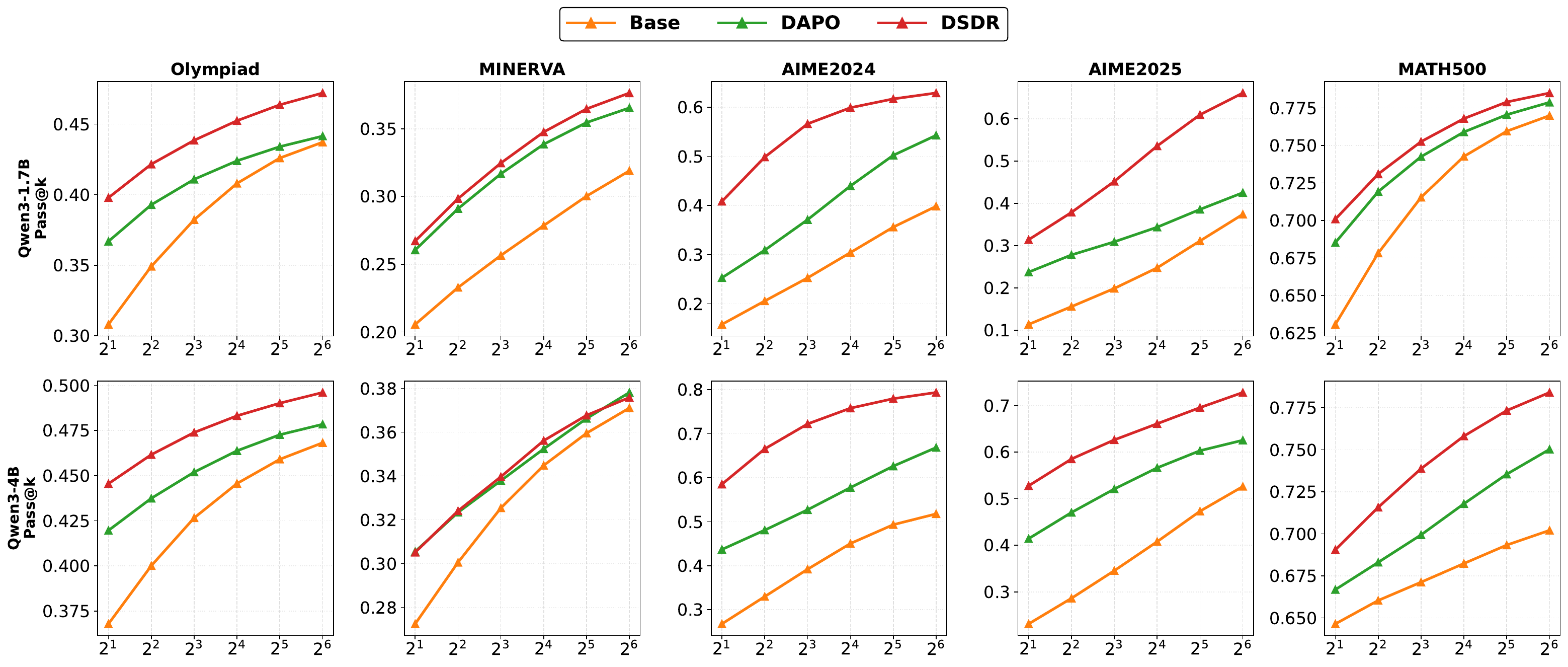}
    \caption{
    Pass@k performance across five benchmarks for both Qwen3-1.7B and Qwen3-4B. The Base models serve as backbones. DSDR consistently outperforms both the Base models and DAPO across all values of $k$..
    }
    \label{fig:passk}
    \end{center}
    \vskip -0.6em
\end{figure*}

\noindent \textbf{Backbone Models.}
For fair comparison, we conduct all experiments on the filtered DAPO-Math-17K~\cite{openr1}, which removes the duplicated samples. Training is performed on four base models with increasing capacity: Qwen2.5-Math-1.5B~\cite{yang2024qwen2}, Qwen3-1.7B and 4B~\cite{yang2025qwen3}. We adopt all-MiniLM-L6-v2~\cite{reimers2019sentence} as a lightweight sentence encoder for extracting response embeddings.

\noindent \textbf{Training Setting.}
We use the same training configuration for all models. The batch size is 256, with 8 rollouts per prompt during policy optimization and a learning rate of 1e-6. The maximum response length is set to 4096 for Qwen2.5-Math-1.5B and 8192 for Qwen3-1.7B and  Qwen3-4B in our experiment setting.

\noindent \textbf{Evaluation Setting.}
Evaluation is conducted on a diverse set of mathematical reasoning benchmarks, including AIME2024~\cite{aime24} and AIME2025, MATH500~\cite{cobbe2021training}, Minerva Math~\cite{lewkowycz2022solving}, and Olympiad-level problems~\cite{he2024olympiadbench}. For each benchmark, we report pass@1, computed from a single rollout per problem, and Avg@16, computed by averaging correctness over 16 independent rollouts. Avg@16 reflects the overall quality and stability of the model’s sampling distribution under a fixed sampling budget. In addition, we evaluate Pass@k for $k \in \{2, 4, ..., 64\}$ by sampling 64 independent rollouts per problem. We compare our method with  GRPO~\cite{shao2024deepseekmath}, DAPO~\cite{yu2025dapo} and the corresponding base models.

\vskip -0.4em
\subsection{Main Results}
\label{sec:main results}
\vskip -0.4em
\noindent \textbf{Overall Performance.}
Our main experimental results are summarized in Table~\ref{tab:main_results}, covering five representative math reasoning benchmarks across three model scales. Overall, DSDR consistently outperforms all compared baselines, including Backbone, GRPO, and DAPO, demonstrating robust and scalable improvements in both Pass@1 and Avg@16 accuracy. On Qwen2.5-Math-1.5B, DSDR achieves the best average performance (25.4 / 25.6), with clear gains on challenging benchmarks such as AIME24, MATH500, and Minerva, where multiple valid reasoning paths exist. These improvements suggest that DSDR better preserves informative learning signals under group-relative optimization by differentiating correct trajectories, mitigating the reward-variance collapse that arises when many rollouts are correct.
The advantage of DSDR becomes more pronounced as model scale increases. On Qwen3-1.7B and Qwen3-4B, DSDR achieves 36.8 / 36.8 and 48.0 / 46.8 average performance respectively, substantially outperforming both GRPO and DAPO at each scale, with especially large margins on AIME24 and AIME25. Importantly, DSDR consistently improves Avg@16 alongside Pass@1, indicating that the gains are not driven by occasional lucky samples but by systematically expanding the diversity of correct reasoning trajectories. Across scales, by promoting exploration exclusively within the correct solution space, DSDR enables more effective and stable exploration in RL-based LLM reasoning, leading to consistent and scalable performance gains.

\begin{figure*}[t]
     \vskip -0.6em
    \begin{center}
    \includegraphics[width=\textwidth]{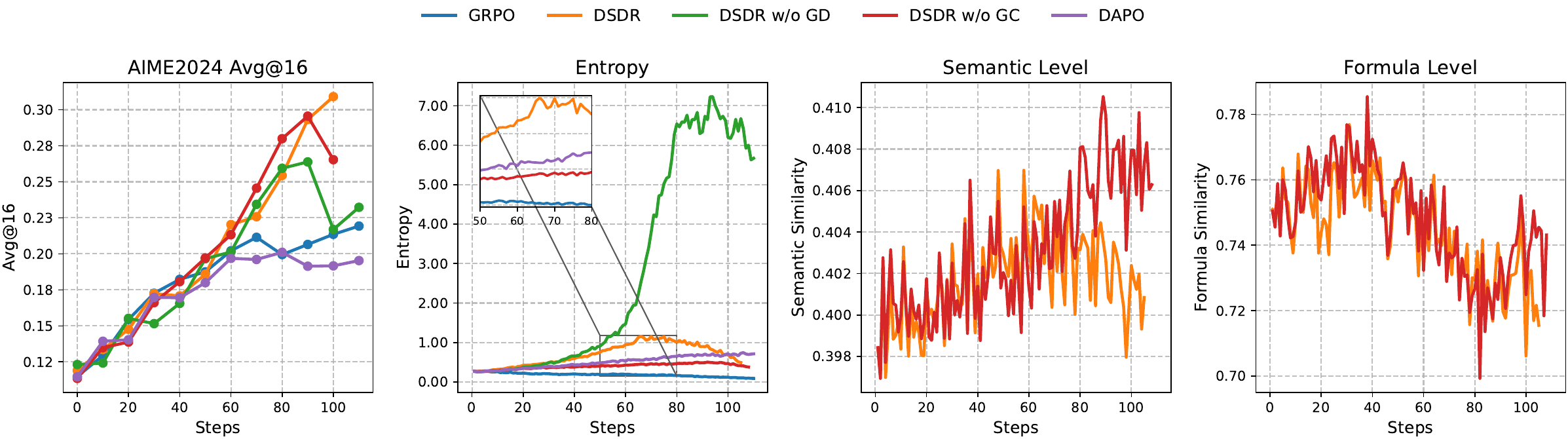}
    \caption{
    Training dynamics across methods conducted on Qwen3-1.7 model. From left to right, we report AIME2024 Avg@16, policy entropy, semantic-level diversity similarity, and formula-level diversity similarity. Results are shown for GRPO, DSDR, DSDR w/o GD, DSDR w/o GC, and DAPO.
    }
    \label{fig:training_dynamics}
    \end{center}
    \vskip -0.6em
\end{figure*}

\begin{figure}[h]
   \vskip -0.8em
    \begin{center}
    \includegraphics[width=\linewidth]{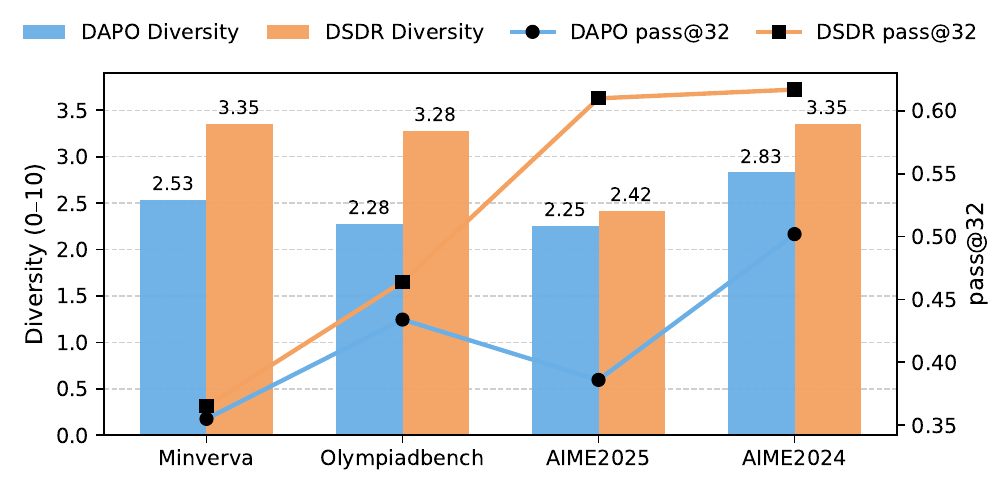}
    \caption{
   We generate 32 test-time rollouts per problem on four benchmarks and evaluate response diversity using an LLM-as-a-Judge (1–10 scale). The figure reports diversity scores and corresponding pass@32 for DAPO and DSDR.
    }
    \label{fig:diversity_pass32}
    \end{center}
   \vskip -0.8em
\end{figure}

\noindent \textbf{Performance on Pass@k Evaluation.}
Figure~\ref{fig:passk} presents the pass@k performance for $k\in \{2,4,...,64\}$ across our method, DAPO and the base model Qwen3-1.7B, qwen3-4B on five benchmarks. Overall, DSDR consistently outperforms the base model and DAPO across a wide range of k, with the most pronounced gains observed on AIME2024, AIME2025, and Olympiad. On these benchmarks, the performance gap between DSDR and the baselines continues to widen as k increases, indicating that DSDR effectively expands the set of correct reasoning trajectories rather than merely sharpening a single dominant solution. On Minerva, where baseline pass@k values are relatively low and correct solutions are sparse, the absolute gains are more modest; nevertheless, DSDR maintains a consistent advantage over both the base model and DAPO across k, suggesting improved exploration even in low-reward regimes. On MATH500, where baseline accuracy is already high, DSDR delivers stable improvements across all k without saturation or degradation at large k. Importantly, DSDR does not exhibit performance drop-offs at high k, highlighting its ability to promote exploration within the correct solution space rather than drifting toward noisy or incorrect samples. These results demonstrate that DSDR yields more reliable and scalable improvements in pass@k.

\subsection{Ablation Study}
\label{sec:ablation study}

We conduct ablation studies on Qwen3-1.7B and 4B. Table~\ref{tab:main_results} reports Avg@16 and Pass@1 for variants without global diversity (GD) and without global-to-local coupling (GC). Removing global diversity causes a clear drop in average performance across benchmarks and model sizes (1.7B and 4B), indicating that trajectory-level diversity is necessary to preserve informative learning signals when many rollouts are correct and verifier rewards saturate. Removing global-to-local coupling also degrades performance, with larger drops on AIME and Olympiad, where discovering multiple valid reasoning strategies is critical. This shows that local entropy alone is insufficient; instead, local regularization must be guided by global distinctiveness to expand underexplored correct trajectories. Overall, the ablations confirm that GD and GC are complementary: GD differentiates correct solutions at the trajectory level, while GC focuses local exploration, jointly enabling the consistent gains of DSDR.

\subsection{Training Dynamics Analysis}
\label{sec:training dynamicss analysis}

Figure~\ref{fig:training_dynamics} shows the training dynamics of DSDR vs. other methods. As training progresses, DSDR achieves consistently higher Avg@16 than GRPO and DAPO on AIME2024, indicating that DSDR improves performance while simultaneously enhancing exploration, preventing the policy from collapsing into a single dominant reasoning pattern. The entropy dynamics further highlight the role of DSDR’s dual-scale design: DSDR w/o GD exhibits a rapid and excessive increase in entropy, reflecting uncontrolled random exploration without global diversity guidance, while DSDR w/o GC, which removes token-level entropy regularization, shows diminishing exploration capacity in later stages as the policy becomes overly concentrated. In contrast, GRPO and DAPO maintain relatively low and flat entropy, suggesting limited exploration. By combining correct-only global diversity (GD) with global-to-local coupling (GC), DSDR achieves a balanced entropy profile that increases exploration without instability. This effect is further reflected in the semantic- and formula-level similarity curves, where DSDR maintains lower semantic similarity and sustained symbolic diversity among rollouts throughout training, indicating that the model continues to explore multiple distinct reasoning trajectories while preserving correctness. Together, these dynamics demonstrate that DSDR enables stable and targeted exploration, preventing both random drift and premature mode collapse.

\subsection{Diversity Analysis}

\label{sec:diversity_analysis}

Figure~\ref{fig:diversity_pass32} compares response diversity and pass@32 across four benchmarks using 32 test-time rollouts. We use GPT-5.2~\cite{singh2025openai} to evaluate the diversity, which aggregates semantic, logical, and formula-level differences, scored on a 1–10 scale. The diversity judge prompt is provided in Appendix~\ref{app:prompt_template}. As shown in the figure, DSDR consistently produces higher diversity scores than DAPO across all datasets, indicating that its generated responses cover a broader range of reasoning strategies rather than collapsing into similar solution patterns. Importantly, these diversity gains are accompanied by higher pass@32 performance, demonstrating that increased diversity does not come at the cost of correctness. It indicates that by applying correct-only global diversity regularization, DSDR encourages distinct correct trajectories, while the global-to-local coupling further expands locally around the most distinctive solutions instead of injecting uniform randomness. As a result, DSDR is able to improve exploration as well as performance, yielding both higher-quality diversity and stronger pass@k performance compared to DAPO.

\begin{figure}[t]
    \vskip -0.6em
    \begin{center}
    \includegraphics[width=\linewidth]{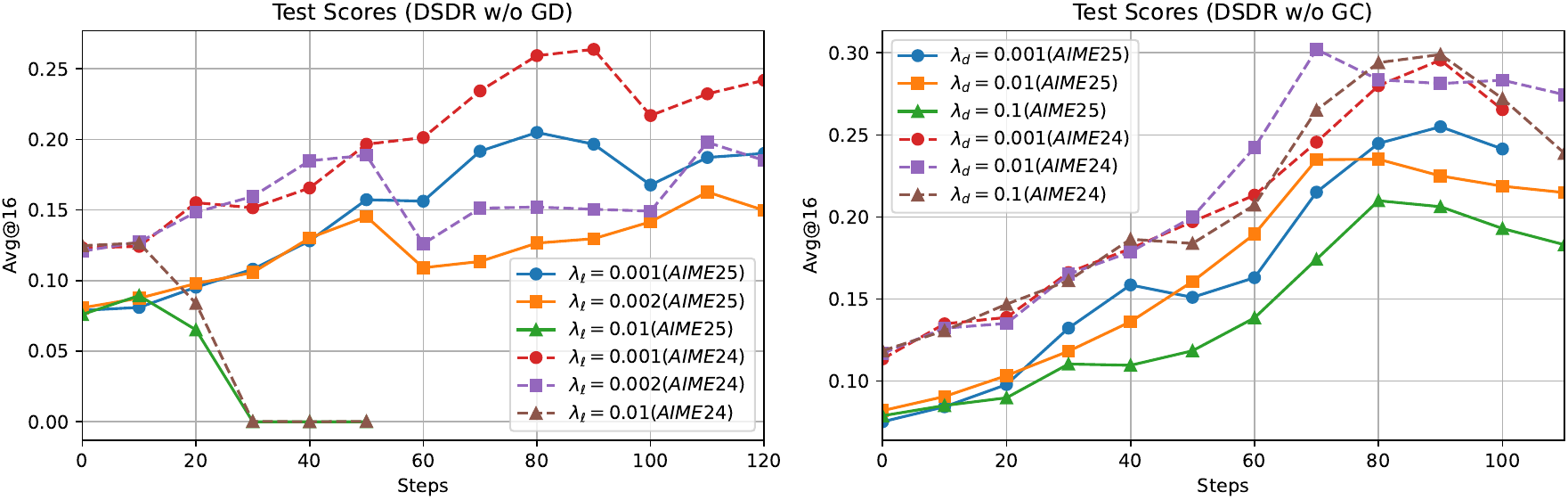}
    \caption{
    Hyperparameter sensitivity of DSDR on Qwen3-1.7B. Left: varying $\lambda_{\ell}$ shows that overly large entropy regularization destabilizes training. Right: $\lambda_d=0.001$ achieves the best and most stable Avg@16 performance on AIME2024/2025.
    }
    \label{fig:hyperparam}
    \end{center}
    \vskip -0.6em
\end{figure}

\subsection{Hyperparameter Sensitivity of $\lambda_{\ell}$ and $\lambda_d$}

\label{sec:diversity_analysis}
We study the sensitivity of DSDR to the local coefficient $\lambda_{\ell}$ and global diversity factor $\lambda_d$ on Qwen3-1.7B (Figure~\ref{fig:hyperparam}). For $\lambda_{\ell} \in \{0.001, 0.002, 0.01\}$, moderate regularization improves performance, while values larger than $0.01$ cause training instability and collapse, indicating that excessive entropy-driven exploration disrupts correctness-aligned optimization. $\lambda_{\ell}=0.001$ yields the most stable and consistently strong results and is used in our main experiments. 
For $\lambda_d$, we observe that $0.001$ achieves the best average performance on AIME2024 and AIME2025 with stable training dynamics, whereas larger values introduce additional variance without consistent gains. We therefore adopt $\lambda_d=0.001$ as the default setting. Overall, DSDR remains stable within a reasonable regularization range.

\section{Conclusion}

In this paper, we introduced \textbf{DSDR}, a correctness-aligned dual-scale diversity regularization framework for RLVR that improves exploration in LLM reasoning. DSDR promotes \emph{global} diversity among \emph{correct} trajectories to sustain multiple solution modes, and applies a \emph{local}, length-invariant token-level entropy regularizer exclusively to correct trajectories to prevent intra-mode entropy collapse. A global-to-local allocation mechanism tightly couples the two scales, focusing local regularization on globally distinctive correct trajectories.
Our analysis shows that bounded local regularization preserves correctness while correct-only global shaping maintains informative learning signals under group-based optimization. Experiments across diverse reasoning benchmarks demonstrate consistent gains in accuracy, pass@k, and training stability, underscoring the importance of coordinating trajectory- and token-level exploration in RLVR for stable and robust policy optimization.

\section{Impact Statements}

This work introduces DSDR, a dual-scale diversity regularization framework for reinforcement learning with verifiable rewards (RLVR) in large language model (LLM) reasoning. By promoting correctness-aligned exploration at both trajectory and token levels, DSDR aims to improve reasoning robustness, stability, and sample efficiency. The proposed approach is intended to support the development of more reliable reasoning-oriented LLMs, with potential benefits for applications that require multi-step decision making and formal reasoning.


\bibliography{main}

@article{liu2024deepseek,
  title={Deepseek-v3 technical report},
  author={Liu, Aixin and Feng, Bei and Xue, Bing and Wang, Bingxuan and Wu, Bochao and Lu, Chengda and Zhao, Chenggang and Deng, Chengqi and Zhang, Chenyu and Ruan, Chong and others},
  journal={arXiv preprint arXiv:2412.19437},
  year={2024}
}

@article{shao2024deepseekmath,
  title={Deepseekmath: Pushing the limits of mathematical reasoning in open language models},
  author={Shao, Zhihong and Wang, Peiyi and Zhu, Qihao and Xu, Runxin and Song, Junxiao and Bi, Xiao and Zhang, Haowei and Zhang, Mingchuan and Li, YK and Wu, Yang and others},
  journal={arXiv preprint arXiv:2402.03300},
  year={2024}
}

@article{yu2025dapo,
  title={Dapo: An open-source llm reinforcement learning system at scale},
  author={Yu, Qiying and Zhang, Zheng and Zhu, Ruofei and Yuan, Yufeng and Zuo, Xiaochen and Yue, Yu and Dai, Weinan and Fan, Tiantian and Liu, Gaohong and Liu, Lingjun and others},
  journal={arXiv preprint arXiv:2503.14476},
  year={2025}
}

@article{comanici2025gemini,
  title={Gemini 2.5: Pushing the frontier with advanced reasoning, multimodality, long context, and next generation agentic capabilities},
  author={Comanici, Gheorghe and Bieber, Eric and Schaekermann, Mike and Pasupat, Ice and Sachdeva, Noveen and Dhillon, Inderjit and Blistein, Marcel and Ram, Ori and Zhang, Dan and Rosen, Evan and others},
  journal={arXiv preprint arXiv:2507.06261},
  year={2025}
}

@article{singh2025openai,
  title={OpenAI GPT-5 System Card},
  author={Singh, Aaditya and Fry, Adam and Perelman, Adam and Tart, Adam and Ganesh, Adi and El-Kishky, Ahmed and McLaughlin, Aidan and Low, Aiden and Ostrow, AJ and Ananthram, Akhila and others},
  journal={arXiv preprint arXiv:2601.03267},
  year={2025}
}

@article{guo2025deepseek,
  title={DeepSeek-R1 incentivizes reasoning in LLMs through reinforcement learning},
  author={Guo, Daya and Yang, Dejian and Zhang, Haowei and Song, Junxiao and Wang, Peiyi and Zhu, Qihao and Xu, Runxin and Zhang, Ruoyu and Ma, Shirong and Bi, Xiao and others},
  journal={Nature},
  volume={645},
  number={8081},
  pages={633--638},
  year={2025},
  publisher={Nature Publishing Group UK London}
}

@article{yang2025qwen3,
  title={Qwen3 technical report},
  author={Yang, An and Li, Anfeng and Yang, Baosong and Zhang, Beichen and Hui, Binyuan and Zheng, Bo and Yu, Bowen and Gao, Chang and Huang, Chengen and Lv, Chenxu and others},
  journal={arXiv preprint arXiv:2505.09388},
  year={2025}
}

@article{liu2025prorl,
  title={Prorl: Prolonged reinforcement learning expands reasoning boundaries in large language models},
  author={Liu, Mingjie and Diao, Shizhe and Lu, Ximing and Hu, Jian and Dong, Xin and Choi, Yejin and Kautz, Jan and Dong, Yi},
  journal={arXiv preprint arXiv:2505.24864},
  year={2025}
}

@article{wu2025invisible,
  title={The invisible leash: Why rlvr may or may not escape its origin},
  author={Wu, Fang and Xuan, Weihao and Lu, Ximing and Liu, Mingjie and Dong, Yi and Harchaoui, Zaid and Choi, Yejin},
  journal={arXiv preprint arXiv:2507.14843},
  year={2025}
}

@article{jiang2025risk,
  title={Risk-Sensitive RL for Alleviating Exploration Dilemmas in Large Language Models},
  author={Jiang, Yuhua and Huang, Jiawei and Yuan, Yufeng and Mao, Xin and Yue, Yu and Zhao, Qianchuan and Yan, Lin},
  journal={arXiv preprint arXiv:2509.24261},
  year={2025}
}

@article{cui2025entropy,
  title={The entropy mechanism of reinforcement learning for reasoning language models},
  author={Cui, Ganqu and Zhang, Yuchen and Chen, Jiacheng and Yuan, Lifan and Wang, Zhi and Zuo, Yuxin and Li, Haozhan and Fan, Yuchen and Chen, Huayu and Chen, Weize and others},
  journal={arXiv preprint arXiv:2505.22617},
  year={2025}
}

@article{yao2025diversity,
  title={Diversity-Aware Policy Optimization for Large Language Model Reasoning},
  author={Yao, Jian and Cheng, Ran and Wu, Xingyu and Wu, Jibin and Tan, Kay Chen},
  journal={arXiv preprint arXiv:2505.23433},
  year={2025}
}

@article{chen2025pass,
  title={Pass@ k training for adaptively balancing exploration and exploitation of large reasoning models},
  author={Chen, Zhipeng and Qin, Xiaobo and Wu, Youbin and Ling, Yue and Ye, Qinghao and Zhao, Wayne Xin and Shi, Guang},
  journal={arXiv preprint arXiv:2508.10751},
  year={2025}
}

@article{shen2025entropy,
  title={On entropy control in llm-rl algorithms},
  author={Shen, Han},
  journal={arXiv preprint arXiv:2509.03493},
  year={2025}
}

@misc{chen2025eepoexplorationenhancedpolicyoptimization,
      title={EEPO: Exploration-Enhanced Policy Optimization via Sample-Then-Forget}, 
      author={Liang Chen and Xueting Han and Qizhou Wang and Bo Han and Jing Bai and Hinrich Schutze and Kam-Fai Wong},
      year={2025},
      eprint={2510.05837},
      archivePrefix={arXiv},
      primaryClass={cs.CL},
      url={https://arxiv.org/abs/2510.05837}, 
}

@article{sullivan2023reward,
  title={Reward scale robustness for proximal policy optimization via DreamerV3 tricks},
  author={Sullivan, Ryan and Kumar, Akarsh and Huang, Shengyi and Dickerson, John and Suarez, Joseph},
  journal={Advances in Neural Information Processing Systems},
  volume={36},
  pages={1352--1362},
  year={2023}
}

@inproceedings{li2023internally,
  title={Internally rewarded reinforcement learning},
  author={Li, Mengdi and Zhao, Xufeng and Lee, Jae Hee and Weber, Cornelius and Wermter, Stefan},
  booktitle={International Conference on Machine Learning},
  pages={20556--20574},
  year={2023},
  organization={PMLR}
}

@article{walder2025pass,
  title={Pass@ K Policy Optimization: Solving Harder Reinforcement Learning Problems},
  author={Walder, Christian and Karkhanis, Deep},
  journal={arXiv preprint arXiv:2505.15201},
  year={2025}
}

@article{agarwal2025unreasonable,
  title={The unreasonable effectiveness of entropy minimization in llm reasoning},
  author={Agarwal, Shivam and Zhang, Zimin and Yuan, Lifan and Han, Jiawei and Peng, Hao},
  journal={arXiv preprint arXiv:2505.15134},
  year={2025}
}

@article{zhang2025right,
  title={Right question is already half the answer: Fully unsupervised llm reasoning incentivization},
  author={Zhang, Qingyang and Wu, Haitao and Zhang, Changqing and Zhao, Peilin and Bian, Yatao},
  journal={arXiv preprint arXiv:2504.05812},
  year={2025}
}

@article{chen2025post,
  title={Post-training Large Language Models for Diverse High-Quality Responses},
  author={Chen, Yilei and Chakraborty, Souradip and Wolf, Lorenz and Paschalidis, Yannis and Pacchiano, Aldo},
  journal={arXiv preprint arXiv:2509.04784},
  year={2025}
}

@misc{li2025jointlyreinforcingdiversityquality,
      title={Jointly Reinforcing Diversity and Quality in Language Model Generations}, 
      author={Tianjian Li and Yiming Zhang and Ping Yu and Swarnadeep Saha and Daniel Khashabi and Jason Weston and Jack Lanchantin and Tianlu Wang},
      year={2025},
      eprint={2509.02534},
      archivePrefix={arXiv},
      primaryClass={cs.CL},
      url={https://arxiv.org/abs/2509.02534}, 
}

@article{hu2025diversity,
  title={Diversity-incentivized exploration for versatile reasoning},
  author={Hu, Zican and Zhang, Shilin and Li, Yafu and Yan, Jianhao and Hu, Xuyang and Cui, Leyang and Qu, Xiaoye and Chen, Chunlin and Cheng, Yu and Wang, Zhi},
  journal={arXiv preprint arXiv:2509.26209},
  year={2025}
}

@article{cobbe2021training,
  title={Training verifiers to solve math word problems},
  author={Cobbe, Karl and Kosaraju, Vineet and Bavarian, Mohammad and Chen, Mark and Jun, Heewoo and Kaiser, Lukasz and Plappert, Matthias and Tworek, Jerry and Hilton, Jacob and Nakano, Reiichiro and others},
  journal={arXiv preprint arXiv:2110.14168},
  year={2021}
}

@article{gandhi2025cognitive,
  title={Cognitive behaviors that enable self-improving reasoners, or, four habits of highly effective stars},
  author={Gandhi, Kanishk and Chakravarthy, Ayush and Singh, Anikait and Lile, Nathan and Goodman, Noah D},
  journal={arXiv preprint arXiv:2503.01307},
  year={2025}
}

@article{hou2025advancing,
  title={Advancing language model reasoning through reinforcement learning and inference scaling},
  author={Hou, Zhenyu and Lv, Xin and Lu, Rui and Zhang, Jiajie and Li, Yujiang and Yao, Zijun and Li, Juanzi and Tang, Jie and Dong, Yuxiao},
  journal={arXiv preprint arXiv:2501.11651},
  year={2025}
}

@article{liu2025ettrl,
  title={Ettrl: Balancing exploration and exploitation in llm test-time reinforcement learning via entropy mechanism},
  author={Liu, Jia and He, ChangYi and Lin, YingQiao and Yang, MingMin and Shen, FeiYang and Liu, ShaoGuo},
  journal={arXiv preprint arXiv:2508.11356},
  year={2025}
}

@article{wu2024progress,
  title={Progress or regress? self-improvement reversal in post-training},
  author={Wu, Ting and Li, Xuefeng and Liu, Pengfei},
  journal={arXiv preprint arXiv:2407.05013},
  year={2024}
}

@article{sheng2025espo,
  title={Espo: Entropy importance sampling policy optimization},
  author={Sheng, Yuepeng and Huang, Yuwei and Liu, Shuman and Zhang, Haibo and Zeng, Anxiang},
  journal={arXiv preprint arXiv:2512.00499},
  year={2025}
}

@article{precup2000eligibility,
  title={Eligibility traces for off-policy policy evaluation},
  author={Precup, Doina and Sutton, Richard S and Singh, Satinder},
  year={2000}
}

@misc{openr1,
    title = {Open R1: A fully open reproduction of DeepSeek-R1},
    url = {https://github.com/huggingface/open-r1},
    author = {{Hugging Face}},
    month = {January},
    year = {2025}
}

@article{yang2024qwen2,
  title={Qwen2. 5-math technical report: Toward mathematical expert model via self-improvement},
  author={Yang, An and Zhang, Beichen and Hui, Binyuan and Gao, Bofei and Yu, Bowen and Li, Chengpeng and Liu, Dayiheng and Tu, Jianhong and Zhou, Jingren and Lin, Junyang and others},
  journal={arXiv preprint arXiv:2409.12122},
  year={2024}
}

@article{reimers2019sentence,
  title={Sentence-bert: Sentence embeddings using siamese bert-networks},
  author={Reimers, Nils and Gurevych, Iryna},
  journal={arXiv preprint arXiv:1908.10084},
  year={2019}
}

@article{pan2022effects,
  title={The effects of reward misspecification: Mapping and mitigating misaligned models},
  author={Pan, Alexander and Bhatia, Kush and Steinhardt, Jacob},
  journal={arXiv preprint arXiv:2201.03544},
  year={2022}
}

@book{cover1999elements,
  title={Elements of information theory},
  author={Cover, Thomas M},
  year={1999},
  publisher={John Wiley \& Sons}
}

@article{schulman2017proximal,
  title={Proximal policy optimization algorithms},
  author={Schulman, John and Wolski, Filip and Dhariwal, Prafulla and Radford, Alec and Klimov, Oleg},
  journal={arXiv preprint arXiv:1707.06347},
  year={2017}
}

@article{jaynes1957information,
  title={Information theory and statistical mechanics},
  author={Jaynes, Edwin T},
  journal={Physical review},
  volume={106},
  number={4},
  pages={620},
  year={1957},
  publisher={APS}
}

@book{boyd2004convex,
  title={Convex optimization},
  author={Boyd, Stephen and Vandenberghe, Lieven},
  year={2004},
  publisher={Cambridge university press}
}

@article{williams1992simple,
  title={Simple statistical gradient-following algorithms for connectionist reinforcement learning},
  author={Williams, Ronald J},
  journal={Machine learning},
  volume={8},
  number={3},
  pages={229--256},
  year={1992},
  publisher={Springer}
}

@article{sutton1999policy,
  title={Policy gradient methods for reinforcement learning with function approximation},
  author={Sutton, Richard S and McAllester, David and Singh, Satinder and Mansour, Yishay},
  journal={Advances in neural information processing systems},
  volume={12},
  year={1999}
}

@article{wan2025srpo,
  title={Srpo: Enhancing multimodal llm reasoning via reflection-aware reinforcement learning},
  author={Wan, Zhongwei and Dou, Zhihao and Liu, Che and Zhang, Yu and Cui, Dongfei and Zhao, Qinjian and Shen, Hui and Xiong, Jing and Xin, Yi and Jiang, Yifan and others},
  journal={arXiv preprint arXiv:2506.01713},
  year={2025}
}

@misc{aime24,
      title={American Invitational Mathematics Examination (AIME) 2024}, 
      author={Zhang, Yifan and Math-AI, Team},
      year={2024},
}

@article{lewkowycz2022solving,
  title={Solving quantitative reasoning problems with language models},
  author={Lewkowycz, Aitor and Andreassen, Anders and Dohan, David and Dyer, Ethan and Michalewski, Henryk and Ramasesh, Vinay and Slone, Ambrose and Anil, Cem and Schlag, Imanol and Gutman-Solo, Theo and others},
  journal={Advances in neural information processing systems},
  volume={35},
  pages={3843--3857},
  year={2022}
}

@inproceedings{he2024olympiadbench,
  title={Olympiadbench: A challenging benchmark for promoting agi with olympiad-level bilingual multimodal scientific problems},
  author={He, Chaoqun and Luo, Renjie and Bai, Yuzhuo and Hu, Shengding and Thai, Zhen and Shen, Junhao and Hu, Jinyi and Han, Xu and Huang, Yujie and Zhang, Yuxiang and others},
  booktitle={Proceedings of the 62nd Annual Meeting of the Association for Computational Linguistics (Volume 1: Long Papers)},
  pages={3828--3850},
  year={2024}
}

@inproceedings{sheng2025hybridflow,
  title={Hybridflow: A flexible and efficient rlhf framework},
  author={Sheng, Guangming and Zhang, Chi and Ye, Zilingfeng and Wu, Xibin and Zhang, Wang and Zhang, Ru and Peng, Yanghua and Lin, Haibin and Wu, Chuan},
  booktitle={Proceedings of the Twentieth European Conference on Computer Systems},
  pages={1279--1297},
  year={2025}
}
\bibliographystyle{icml2026}

\newpage
\appendix
\onecolumn
\section{Implementation Details}
\label{app:implementation_details}

We provide more details for experiments in Section~\ref{sec:experiments}.
We provide additional experimental details in Section~\ref{sec:experiments}.
All models are trained using the \textsc{verl} framework~\cite{sheng2025hybridflow} and deployed on 8$\times$ NVIDIA A100 GPUs (40GB). Table~\ref{tab:training_details} and Table~\ref{tab:evaluation_settings} summarize the training and evaluation hyperparameters.
Unless otherwise specified, we adopt a rollout group size of $n=8$, a learning rate of $1\times10^{-6}$, and binary verifier rewards. We conduct hyperparameter sweeps over the global diversity scaling factor $\lambda_d \in \{0.001, 0.01, 0.1\}$, the local regularization coefficient $\lambda_{\ell} \in \{0.001, 0.002, 0.01\}$, and the coupling temperature $\tau \in \{1, 5, 10\}$. Empirically, we find that $\lambda_d=0.001$, $\lambda_{\ell}=0.001$, and $\tau=5$ consistently yield the best or near-best performance across benchmarks; these values are therefore used as defaults in all reported experiments.

\begin{table*}[h]
\begin{center}
\caption{Summary of training details.}
\label{tab:training_details}
\begin{tabular*}{\textwidth}{@{\extracolsep{\fill}} l l}
\toprule
\multicolumn{2}{l}{\textbf{Training Settings}} \\
\midrule
Hardware & 8$\times$A100 GPUs (40GB) \\
Base models & Qwen2.5-Math-1.5B / Qwen3-1.7B, 4B \\
Training dataset & open-r1/DAPO-Math-17K-Processed \\
Max response length & 4096 / 8192 \\
Batch / mini-batch size & 256 / 16 \\
Rollout group size $n$ & 8 \\
Learning rate & $1 \times 10^{-6}$ \\
Temperature (training) & 1.0 \\
Clip range $(\epsilon_{\text{low}}, \epsilon_{\text{high}})$ & (0.2, 0.28) / (0.2, 0.2) for GRPO only \\
Reward type & Binary reward \\
Coupling temperature parameter $\tau$  &  $\{1, 5, 10\}$\\
Reward scaling factor $\lambda_{d}$ & $\{0.001, 0.01, 0.1 \}$ \\
Regularization coefficient $\lambda_{\ell}$ & $\{0.001, 0.002, 0.01\}$ \\
\bottomrule
\end{tabular*}
\end{center}
\vskip -0.1in
\end{table*}

\begin{table*}[h]
\begin{center}
\caption{Evaluation settings.}
\label{tab:evaluation_settings}
\begin{tabular*}{\textwidth}{@{\extracolsep{\fill}} l l}
\toprule
\multicolumn{2}{l}{\textbf{Evaluation settings}} \\
\midrule
Max response length & 4096 / 8192 \\
Top-$p$ (eval) & 0.9 \\
Temperature (eval) & 0.1 for Pass@1; 0.7 for Avg@16 \\
\bottomrule
\end{tabular*}
\end{center}
\vskip -0.1in
\end{table*}

\section{Additional Results}
\label{app:add_results}

\subsection{Additional Training Dynamics Analysis}
\label{app:add_training_dynamics}

As shown in figure~\ref{fig:add_training_dynamics}, DSDR steadily improves both Avg@16 and Pass@16 over training, indicating that performance gains are accompanied by sustained exploration rather than convergence to a narrow solution mode. In contrast, GRPO and DAPO display slower growth and earlier saturation, suggesting limited ability to expand the set of correct solutions. And removing global diversity (w/o GD) results in large fluctuations in the policy-gradient loss, reflecting unstable learning when trajectory-level differentiation among correct rollouts is absent. on the other hand, without global-to-local coupling (w/o GC), leads to weaker late-stage improvements and a sharp increase in the clip ratio, indicating overly aggressive updates and reduced robustness as token-level exploration diminishes. By jointly enforcing correct-only global diversity and diversity-guided local regularization, DSDR maintains controlled updates and effective exploration throughout training, which translates into more reliable and higher final performance.

\begin{figure*}[t]
    \vskip 0.1in
    \begin{center}
    \includegraphics[width=\textwidth]{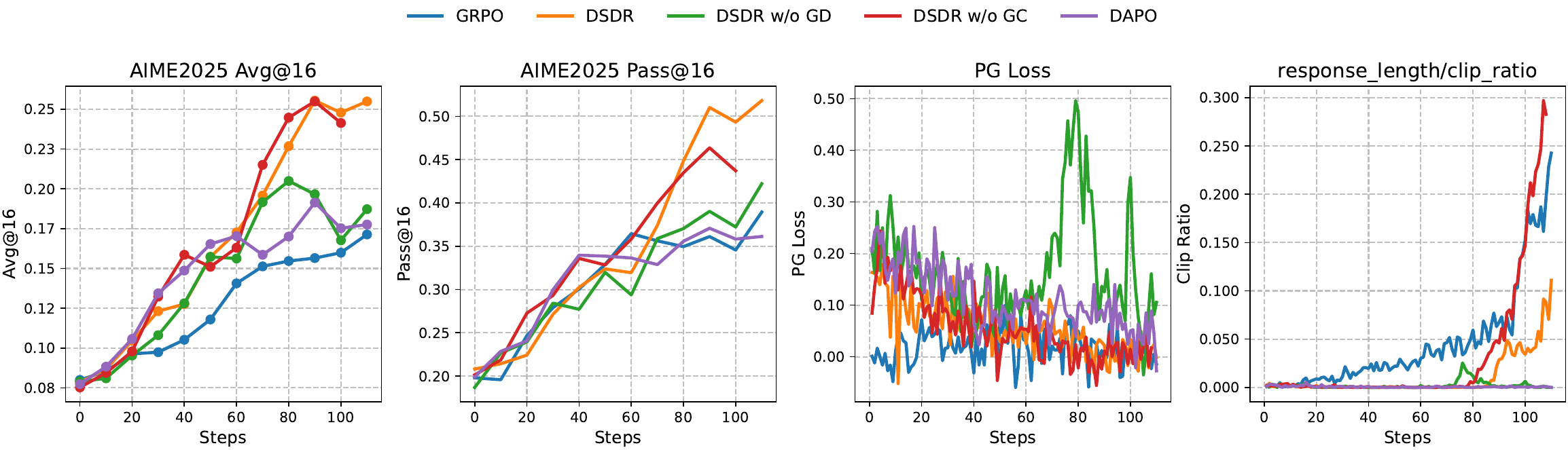}
    \caption{
    Training dynamics across methods conducted on Qwen3-1.7B model. From left to right, we show Avg@16 and Pass@16 for AIME2025, PG loss, and response length/clip ratio, comparing GRPO, DSDR, DSDR w/o GD, DSDR w/o GC, and DAPO.
    }
    \label{fig:add_training_dynamics}
    \end{center}
\end{figure*}

\subsection{Case Study}
\label{app:case_study}

In this section, we present generated samples during testing. One sample with 16 test-time rollouts, DSDR achieves 7 correct solutions, compared to only 2 for DAPO. We show two example responses generated by DSDR, which arrive at the correct answer through distinct reasoning processes, illustrating its ability to explore multiple valid solution paths. In contrast, we present two responses generated by DAPO, which exhibit more limited diversity. This demonstrates that DSDR improves exploration without sacrificing solution fidelity: by maintaining controlled diversity across rollouts, the model avoids mode collapse while consistently discovering correct reasoning paths, leading to stronger overall performance.

\subsubsection{Samples generated by DSDR}
\label{app:sameples_from_DSDR}

Two generated samples by DSDR are shown below, and the yellow boxes highlight the different solution strategies used to solve this problem.

\begin{tcolorbox}[
  enhanced,
  breakable,
  width=\textwidth,
  enlarge left by=-\parindent,
  enlarge right by=-\parindent,
  colback=orange!6,
  colframe=orange!65!black,
  arc=4pt,
  boxrule=0.9pt,
  left=12pt,
  right=12pt,
  top=10pt,
  bottom=10pt,
  before skip=10pt,
  after skip=12pt,
  title={\textbf{Generated Sample: Answer 1}}
]

\textbf{Problem.}

Let $N$ be the greatest four-digit positive integer with the property that whenever one of its digits is changed to $1$, the resulting number is divisible by $7$. Let $Q$ and $R$ be the quotient and remainder, respectively, when $N$ is divided by $1000$. Find $Q+R$. \textbf{Ground Truth:} \textcolor{red}{\boxed{699}}

\vskip 0.8em
\textbf{Answer.}

We are given a four-digit number $ N $ such that whenever one of its digits is changed to 1, the resulting number is divisible by 7. We are to find $ Q + R $, where $ N = 1000Q + R $, and $ Q $ and $ R $ are the quotient and remainder when $ N $ is divided by 1000.

\subsubsection*{Step 1: Understand the condition}

Let’s consider all possible changes to the digits of $ N $, and the resulting number must be divisible by 7.

Let $ N = ABCD $, where $ A, B, C, D $ are digits (each from 0 to 9, with $ A \neq 0 $).

We must ensure that \textbf{for every digit position} (i.e., for each digit $ A, B, C, D $), changing that digit to 1 results in a number divisible by 7.

This means:

- Changing the first digit $ A $ to 1: $ 1BCD $ must be divisible by 7.
- Changing the second digit $ B $ to 1: $ AB1D $ must be divisible by 7.
- Changing the third digit $ C $ to 1: $ ABD1 $ must be divisible by 7.
- Changing the fourth digit $ D $ to 1: $ ABC1 $ must be divisible by 7.

So we have four divisibility conditions:
\[
1BCD \equiv 0,\quad AB1D \equiv 0,\quad ABD1 \equiv 0,\quad ABC1 \equiv 0 \pmod{7}.
\]

\subsubsection*{Step 2: Represent the numbers algebraically}

Let $ N = 1000A + 100B + 10C + D $.

The four conditions become:
\[
\begin{aligned}
1000 + 100B + 10C + D &\equiv 0 \pmod{7}, \\
1000A + 100 + 10C + D &\equiv 0 \pmod{7}, \\
1000A + 100B + 10 + D &\equiv 0 \pmod{7}, \\
1000A + 100B + 10C + 1 &\equiv 0 \pmod{7}.
\end{aligned}
\]

\subsubsection*{Step 3: Use modular arithmetic to simplify}

Using
\[
1000 \equiv 6,\quad 100 \equiv 2,\quad 10 \equiv 3 \pmod{7},
\]
we obtain four congruences, which are then manipulated to derive relations among $A,B,C,D$.

\subsubsection*{Step 4: Solving relations}

Through successive subtraction and modular inversion, we find:
\[
A = 5,\quad B = 6.
\]

\begin{tcolorbox}[
  colback=yellow!85,
  colframe=yellow!60!black,
  arc=2pt,
  boxrule=0pt,
  left=4pt,
  right=4pt,
  top=4pt,
  bottom=4pt
]
Solving the remaining congruences yields two admissible digit pairs:
\[
(C, D) \in \{(2,4), (9,4)\}.
\]
\end{tcolorbox}

\begin{tcolorbox}[
  colback=yellow!85,
  colframe=yellow!60!black,
  arc=2pt,
  boxrule=0pt,
  left=4pt,
  right=4pt,
  top=4pt,
  bottom=4pt
]
Each candidate is then explicitly checked by substituting back into all four digit-replacement conditions to verify divisibility by $7$.
\end{tcolorbox}

Both $5624$ and $5694$ satisfy the conditions. Since the problem asks for the greatest such number, we select
\[
N = 5694.
\]

Thus,
\[
Q = 5,\quad R = 694,\quad Q+R = \boxed{699}.
\]

\subsubsection*{Final Answer}
\[
\boxed{699}
\]

\end{tcolorbox}

\begin{tcolorbox}[
  enhanced,
  breakable,
  width=\textwidth,
  enlarge left by=-\parindent,
  enlarge right by=-\parindent,
  colback=orange!6,
  colframe=orange!65!black,
  arc=4pt,
  boxrule=0.9pt,
  left=12pt,
  right=12pt,
  top=10pt,
  bottom=10pt,
  before skip=12pt,
  after skip=14pt,
  title={\textbf{Generated Sample: Answer 2}}
]

\textbf{Problem.}

Let $N$ be the greatest four-digit positive integer with the property that whenever one of its digits is changed to $1$, the resulting number is divisible by $7$. Let $Q$ and $R$ be the quotient and remainder, respectively, when $N$ is divided by $1000$. Find $Q+R$. \textbf{Ground Truth:} \textcolor{red}{\boxed{699}}

\vskip 0.8em
\textbf{Answer.}

We are given a four-digit number $ N $, such that whenever one of its digits is changed to 1, the resulting number is divisible by 7. We are to find $ Q + R $, where $ Q $ is the quotient and $ R $ is the remainder when $ N $ is divided by 1000.

\subsubsection*{Step 1: Understanding the Problem}

We need to find the \textbf{greatest} four-digit number $ N $ such that \textbf{any single digit change to 1} results in a number divisible by 7.

Let $ N = ABCD $, where $ A, B, C, D $ are digits with $ A \neq 0 $.

For each digit position, changing that digit to 1 must yield a number divisible by 7.

\subsubsection*{Step 2: Formulating Conditions}

\textbf{Case 1:} Change the thousands digit $ A $ to 1:
\[
1000 + BCD \equiv 0 \pmod{7}
\Rightarrow BCD \equiv 1 \pmod{7}.
\]

\textbf{Case 2:} Change the hundreds digit $ B $ to 1:
\[
1000A + 100 + 10C + D \equiv 0 \pmod{7}
\Rightarrow 6A + 3C + D \equiv 5 \pmod{7}.
\]

\textbf{Case 3:} Change the tens digit $ C $ to 1:
\[
1000A + 100B + 10 + D \equiv 0 \pmod{7}
\Rightarrow 6A + 2B + D \equiv 4 \pmod{7}.
\]

\textbf{Case 4:} Change the units digit $ D $ to 1:
\[
1000A + 100B + 10C + 1 \equiv 0 \pmod{7}
\Rightarrow 6A + 2B + 3C \equiv 6 \pmod{7}.
\]

\subsubsection*{Step 3: Summarize the Conditions}

Using $100 \equiv 2$, $10 \equiv 3 \pmod{7}$, we obtain:
\[
2B + 3C + D \equiv 1 \pmod{7}.
\]

Together with the other three congruences, we form a system in $A,B,C,D$.

\subsubsection*{Step 4: Solve the System of Congruences}

By subtracting equations and eliminating variables, we obtain:
\[
3A - B \equiv 2 \pmod{7}, \qquad
3C - D \equiv 2 \pmod{7}, \qquad
2B - 3C \equiv 6 \pmod{7}.
\]

\subsubsection*{Step 5: Express Variables}

Solving step-by-step yields:
\[
A = 5, \quad B = 6,
\]
and
\[
C \in \{2, 9\}, \quad D = 4.
\]

\begin{tcolorbox}[
  colback=yellow!85,
  colframe=yellow!60!black,
  arc=2pt,
  boxrule=0pt,
  left=4pt,
  right=4pt,
  top=4pt,
  bottom=4pt
]
This branching step yields two admissible candidates:
\[
N \in \{5624,\; 5694\}.
\]
\end{tcolorbox}

\subsubsection*{Step 8: Verification}

All four digit-replacement conditions are explicitly checked and satisfied for
\[
N = 5694.
\]

\subsubsection*{Step 9: Compute $Q+R$}

Since $5694 = 1000 \cdot 5 + 694$, we have:
\[
Q = 5,\quad R = 694,\quad Q+R = \boxed{699}.
\]

\subsubsection*{Final Answer}
\[
\boxed{699}
\]

\end{tcolorbox}

\subsubsection{Samples generated by DAPO}
\label{app:samples_generated_DAPO}

Two samples generated by DAPO are shown below, and the red boxes highlight the erroneous solution strategies that lead to failure.

\begin{tcolorbox}[
  enhanced,
  breakable,
  width=\textwidth,
  enlarge left by=-\parindent,
  enlarge right by=-\parindent,
  colback=orange!6,
  colframe=orange!65!black,
  arc=4pt,
  boxrule=0.9pt,
  left=12pt,
  right=12pt,
  top=10pt,
  bottom=10pt,
  before skip=12pt,
  after skip=14pt,
  title={\textbf{Generated Sample: Answer 1}}
]

\textbf{Problem.}

Let $N$ be the greatest four-digit positive integer with the property that whenever one of its digits is changed to $1$, the resulting number is divisible by $7$. Let $Q$ and $R$ be the quotient and remainder, respectively, when $N$ is divided by $1000$. Find $Q+R$. \textbf{Ground Truth:} \textcolor{red}{\boxed{699}}

\vskip 0.8em
\textbf{Answer.}

We are given a \textbf{four-digit number $ N $} such that \textbf{any change of one digit to 1 results in a number divisible by 7}.

Let’s denote $ N $ as a four-digit number:
\[
N = 1000a + 100b + 10c + d
\]
Where:
\begin{itemize}
  \item $ a, b, c, d $ are digits (integers from 0 to 9, with $ a \neq 0 $ since $ N $ is a four-digit number),
  \item $ a \in \{1,2,\dots,9\} $, $ b,c,d \in \{0,1,\dots,9\} $.
\end{itemize}

\subsubsection*{Step 1: Understanding the condition}

The problem says: \textbf{Whenever one digit is changed to 1, the resulting number is divisible by 7}.

So, there are four different cases:
\begin{enumerate}
  \item Change $ a $ to 1: $ 1000 + 100b + 10c + d $
  \item Change $ b $ to 1: $ 1000a + 100 + 10c + d $
  \item Change $ c $ to 1: $ 1000a + 100b + 10 + d $
  \item Change $ d $ to 1: $ 1000a + 100b + 10c + 1 $
\end{enumerate}

Each must be divisible by 7, giving:
\[
\begin{aligned}
1000 + 100b + 10c + d &\equiv 0 \pmod{7}, \\
1000a + 100 + 10c + d &\equiv 0 \pmod{7}, \\
1000a + 100b + 10 + d &\equiv 0 \pmod{7}, \\
1000a + 100b + 10c + 1 &\equiv 0 \pmod{7}.
\end{aligned}
\]

Using
\[
1000 \equiv 6,\quad 100 \equiv 2,\quad 10 \equiv 3 \pmod{7},
\]
we obtain:
\[
\begin{aligned}
2b + 3c + d &\equiv 1 \pmod{7} \quad (1)\\
6a + 3c + d &\equiv 5 \pmod{7} \quad (2)\\
6a + 2b + d &\equiv 4 \pmod{7} \quad (3)\\
6a + 2b + 3c &\equiv 6 \pmod{7} \quad (4)
\end{aligned}
\]

We now attempt to solve this system of congruences.

\begin{tcolorbox}[
  colback=red!15,
  colframe=red!60!black,
  arc=2pt,
  boxrule=0pt,
  left=4pt,
  right=4pt,
  top=4pt,
  bottom=4pt
]
\textbf{Logical error starts here (unsupported assumption).}

Assume $ a = 1 $ (the smallest possible digit, since $ N $ is a four-digit number), then try values of $ b, c, d $ that satisfy the congruences.
\end{tcolorbox}

From this point onward, the reasoning proceeds by ad-hoc guessing and trial-based arguments that are not logically implied by the modular system above.

Eventually, the solution concludes:
\[
\boxed{Q = 111,\quad R = 1,\quad Q+R = 112}
\]

\subsubsection*{Final Answer}
\[
\boxed{112}
\]

\end{tcolorbox}

\begin{tcolorbox}[
  enhanced,
  breakable,
  width=\textwidth,
  enlarge left by=-\parindent,
  enlarge right by=-\parindent,
  colback=orange!6,
  colframe=orange!65!black,
  arc=4pt,
  boxrule=0.9pt,
  left=12pt,
  right=12pt,
  top=10pt,
  bottom=10pt,
  before skip=12pt,
  after skip=14pt,
  title={\textbf{Generated Sample: Answer 2}}
]

\textbf{Problem.}

Let $N$ be the greatest four-digit positive integer with the property that whenever one of its digits is changed to $1$, the resulting number is divisible by $7$. Let $Q$ and $R$ be the quotient and remainder, respectively, when $N$ is divided by $1000$. Find $Q+R$. \textbf{Ground Truth:} \textcolor{red}{\boxed{699}}

\vskip 0.8em
\textbf{Answer.}

We are asked to find the \textbf{greatest four-digit number $ N $} such that \textbf{if any one of its digits is changed to 1}, the resulting number is divisible by 7.
Then, given $ N = 1000Q + R $, we are to find $ Q + R $.

\subsubsection*{Step 1: Understand the condition}

Let’s denote $ N $ as a four-digit number:
\[
N = 1000a + 100b + 10c + d
\]
where $ a,b,c,d \in \{0,1,\dots,9\} $, with $ a \neq 0 $.

Changing each digit to 1 gives the conditions:
\[
\begin{aligned}
1000 + 100b + 10c + d &\equiv 0 \pmod{7},\\
1000a + 100 + 10c + d &\equiv 0 \pmod{7},\\
1000a + 100b + 10 + d &\equiv 0 \pmod{7},\\
1000a + 100b + 10c + 1 &\equiv 0 \pmod{7}.
\end{aligned}
\]

Using
\[
1000 \equiv 6,\quad 100 \equiv 2,\quad 10 \equiv 3 \pmod{7},
\]
we obtain:
\[
\begin{aligned}
2b + 3c + d &\equiv 1 \pmod{7},\\
6a + 3c + d &\equiv 5 \pmod{7},\\
6a + 2b + d &\equiv 4 \pmod{7},\\
6a + 2b + 3c &\equiv 6 \pmod{7}.
\end{aligned}
\]

\begin{tcolorbox}[
  colback=red!15,
  colframe=red!60!black,
  arc=2pt,
  boxrule=0pt,
  left=4pt,
  right=4pt,
  top=4pt,
  bottom=4pt
]
\textbf{Logical error starts here (invalid constraint assumption).}

Let’s suppose that all four expressions are congruent to the same value $ x \pmod{7} $, and in particular, assume
\[
x = 0 \quad (\text{i.e., } N \equiv 0 \pmod{7}).
\]
\end{tcolorbox}

From this point onward, the reasoning proceeds by enumerating digit values under the unsupported assumption that all four congruences must equal the same residue class, which is not implied by the original system.

The solution eventually concludes:
\[
Q = 9,\quad R = 888,\quad Q+R = 897.
\]

\subsubsection*{Final Answer}
\[
\boxed{897}
\]

\end{tcolorbox}

\subsection{Prompt Template}
\label{app:prompt_template}

\begin{tcolorbox}[
    colback=blue!3,
    colframe=blue!55!black,
    width=\textwidth,
    enlarge left by=-\parindent,
    enlarge right by=-\parindent,
    arc=4pt,
    boxrule=0.8pt,
    title=\textbf{Prompt Template for Diversity Scoring},
    fonttitle=\bfseries,
]

\textbf{System:}
You are a strict but fair grader. You will be given one math problem and
\textbf{32 independent rollouts (solutions)}.
Your task is to score diversity across rollouts.

\vskip 0.5em
Evaluate the following three dimensions on a \textbf{0--10 scale}:

\begin{itemize}[leftmargin=1.5em]
    \item \textbf{logic\_diversity}: differences in reasoning flow or step order
    \item \textbf{formula\_diversity}: differences in formulas, identities, or techniques used
    \item \textbf{semantic\_diversity}: differences in overall semantic approach or framing
\end{itemize}

\vskip 0.5em
Then output \textbf{overall diversity} as the average of the three scores
(\textbf{0--10}, rounded to \textbf{one decimal}).

\vskip 0.5em
Be objective and do not prefer any model.
\textbf{Output JSON only.}

\end{tcolorbox}

\section{Theoretical Analysis of DSDR}
\label{app:dsdr_theory}

This appendix part provides formal justification for the key design choices in
DSDR (Sec.~\ref{sec:dsdr}): (i) the \emph{correct-only} global diversity reward
(Eq.~\eqref{eq:augmented_reward}), (ii) the global-to-local coupling weights
(Eq.~\eqref{eq:diversity_softmax}), and (iii) the \emph{positive-only}, length-invariant
token-level entropy regularizer (Eq.~\eqref{eq:local_objective}).

\subsection{Setup, Notation, and Standing Assumptions}
\label{app:setup}

We consider RLVR with verifiable reward on prompts $q\sim\mathcal D$. A policy $\pi(\cdot\mid q)$ induces an output random variable $O$ (a sequence of tokens)
and a verifier reward $R(q,O)\in\{0,1\}$.
For group-based training (GRPO-style), for each prompt $q$ we sample a group
$\{o_i\}_{i=1}^G$ and obtain rewards $\{r_i\}_{i=1}^G$ with $r_i=R(q,o_i)$.

\paragraph{Augmented Reward (Correct-Only Global Diversity Reward)}
For each rollout $o_i$, we compute a diversity score $d(o_i)$ and use a clipped score
\begin{equation}
\bar d_i=\operatorname{clip}\!\big(d(o_i);0,\sigma_d\big)\in[0,\sigma_d].
\label{eq:app_dbar}
\end{equation}
The augmented reward used for group advantage computation is
\begin{equation}
\tilde r_i = r_i + \lambda_d\,\bar d_i\,\mathbbm 1(r_i=1),
\qquad \lambda_d\ge 0.
\label{eq:app_rtilde}
\end{equation}

\paragraph{Group Normalization}
Define the empirical mean and standard deviation of $\{\tilde r_i\}_{i=1}^G$:
\begin{equation}
\mu_{\tilde r}=\frac{1}{G}\sum_{i=1}^G \tilde r_i,
\qquad
\sigma_{\tilde r}=\sqrt{\frac{1}{G}\sum_{i=1}^G (\tilde r_i-\mu_{\tilde r})^2}.
\label{eq:app_mean_std}
\end{equation}
In practice we use a stabilized standard deviation $\sigma_{\tilde r,\varepsilon}>0$, e.g.,
$\sigma_{\tilde r,\varepsilon}=\sigma_{\tilde r}+\varepsilon$ or
$\sigma_{\tilde r,\varepsilon}=\sqrt{\sigma_{\tilde r}^2+\varepsilon^2}$
with $\varepsilon>0$. The normalized advantages are
\begin{equation}
\tilde A_i=\frac{\tilde r_i-\mu_{\tilde r}}{\sigma_{\tilde r,\varepsilon}}.
\label{eq:app_adv}
\end{equation}
All results below remain valid under either stabilization choice as long as $\varepsilon>0$.

\paragraph{Entropy Conventions.}
The action space at each decoding step is the vocabulary $\mathcal V$, assumed finite.
All logarithms are natural logs. Standard entropy identities and bounds follow
classical information theory~\citep{cover1999elements}.

\paragraph{Correctness Gap Assumption}
For Proposition~\ref{prop:correctness_preservation}, we assume a strict gap
\begin{equation}
\Delta := J_R^\star - \overline J_R > 0,
\label{eq:app_gap}
\end{equation}
where $J_R^\star$ is optimal correctness and $\overline J_R$ is the best correctness strictly below optimal.
This separation commonly holds when the set of achievable correctness values is discrete under a restricted policy class.

\subsection{Inter-/Intra-mode Decomposition (Formal Complementarity)}
\label{app:inter_intra}

DSDR introduces \emph{two} diversity mechanisms operating at different scales:
a \emph{global} (sequence-level) preference among correct trajectories
(Eq.~\eqref{eq:augmented_reward}) and a \emph{local} (token-level) positive-only entropy
regularizer (Eq.~\eqref{eq:local_objective}). This subsection provides an
information-theoretic lens clarifying why these two terms are complementary rather
than redundant: global shaping targets \emph{inter-mode} coverage (deep exploration across
distinct reasoning modes), while local entropy targets \emph{intra-mode} thickening
(maintaining non-collapsed variability within a mode).

We view generation as a mixture over latent ``reasoning modes'' $Z$ (e.g., distinct solution strategies).
For each prompt $q$, suppose
\begin{equation}
Z \sim p(z\mid q),\qquad O\sim p(o\mid z,q).
\label{eq:app_latent_gen}
\end{equation}

\begin{lemma}[Inter-/intra-mode entropy decomposition]
\label{lem:inter_intra}
For any fixed prompt $q$,
\begin{equation}
H(O\mid q)=I(O;Z\mid q)+H(O\mid Z,q),
\label{eq:app_inter_intra}
\end{equation}
where $I(O;Z\mid q)$ is conditional mutual information.
\end{lemma}

\begin{proof}
This is a standard identity~\citep{cover1999elements}. By definition,
$ I(O;Z\mid q)=H(O\mid q)-H(O\mid Z,q)$.
Rearranging yields~\eqref{eq:app_inter_intra}.
\end{proof}

\paragraph{Interpretation of DSDR.}
$I(O;Z\mid q)$ captures \emph{inter-mode} diversity (deep exploration across reasoning modes),
while $H(O\mid Z,q)$ captures \emph{intra-mode} diversity (variation within a mode). DSDR’s global shaping
primarily promotes inter-mode coverage, while local positive-only entropy discourages intra-mode collapse.

\subsection{Correctness Preservation Under Bounded Local Regularization}
\label{app:correctness}

The local objective in DSDR (Eq.~\eqref{eq:local_objective}) increases token-level entropy
\emph{only along correct trajectories} and is length-normalized to avoid incentivizing
longer outputs. A natural concern is whether adding this regularizer could sacrifice
verifiable correctness. This subsection shows that, at the population level,
a \emph{sufficiently small} local regularization weight cannot make a lower-correctness
policy optimal; it can only break ties among correctness-optimal policies.

\paragraph{Entropy Upper Bound for Length-invariant Token Entropy.}
For any categorical distribution over $|\mathcal V|$ tokens, entropy is maximized by the uniform distribution,
so $0\le \mathcal H(\pi(\cdot\mid s))\le \log|\mathcal V|$~\citep{cover1999elements}.
Define the length-invariant per-token conditional entropy functional
\begin{equation}
J_H(\pi)
=
\mathbb E\!\left[\frac{1}{T}\sum_{t=1}^T \mathcal H\big(\pi(\cdot\mid q,o_{<t})\big)\right].
\label{eq:app_JH}
\end{equation}
Then
\begin{equation}
0\le J_H(\pi)\le H_{\max}:=\log|\mathcal V|.
\label{eq:app_JH_bound}
\end{equation}
The same upper bound also holds for the positive-only variant
$J_H^+(\pi)=\mathbb E[\mathbbm 1(R(q,O)=1)\cdot \frac{1}{T}\sum_t \mathcal H(\cdot)]$
since $\mathbbm 1(\cdot)\le 1$.

\begin{proposition}[Correctness preservation under bounded $\lambda_\ell$]
\label{prop:correctness_preservation}
Define the population correctness objective
\begin{equation}
J_R(\pi)=\mathbb E_{q\sim\mathcal D,\;o\sim \pi(\cdot\mid q)}[R(q,o)]\in[0,1],
\label{eq:app_JR}
\end{equation}
and let $J_R^\star=\sup_\pi J_R(\pi)$. Define the best correctness among strictly suboptimal policies
\begin{equation}
\overline J_R=\sup_{\pi:\;J_R(\pi)<J_R^\star} J_R(\pi),
\qquad
\Delta:=J_R^\star-\overline J_R>0.
\label{eq:app_JR_bar}
\end{equation}
Let $J_H(\pi)$ be any regularizer satisfying $0\le J_H(\pi)\le H_{\max}$ (e.g.,~\eqref{eq:app_JH}).
Consider the regularized objective
\begin{equation}
J_{\mathrm{reg}}(\pi)=J_R(\pi)+\lambda_\ell J_H(\pi),
\qquad \lambda_\ell\ge 0.
\label{eq:app_Jreg}
\end{equation}
If $\lambda_\ell<\Delta/H_{\max}$, then every maximizer $\pi^\star_{\mathrm{reg}}\in\arg\max_\pi J_{\mathrm{reg}}(\pi)$ is correctness-optimal:
\begin{equation}
J_R(\pi^\star_{\mathrm{reg}})=J_R^\star.
\label{eq:app_correct_opt}
\end{equation}
\end{proposition}

\begin{proof}
Take any correctness-suboptimal policy $\pi$ with $J_R(\pi)\le \overline J_R$.
Using $J_H(\pi)\le H_{\max}$,
\begin{equation}
J_{\mathrm{reg}}(\pi)=J_R(\pi)+\lambda_\ell J_H(\pi)\le \overline J_R+\lambda_\ell H_{\max}.
\label{eq:app_subopt_upper}
\end{equation}
Take any correctness-optimal policy $\pi^\star$ with $J_R(\pi^\star)=J_R^\star$.
Since $J_H(\pi^\star)\ge 0$,
\begin{equation}
J_{\mathrm{reg}}(\pi^\star)=J_R^\star+\lambda_\ell J_H(\pi^\star)\ge J_R^\star.
\label{eq:app_opt_lower}
\end{equation}
If $\lambda_\ell<\Delta/H_{\max}$, then
\begin{equation}
\overline J_R+\lambda_\ell H_{\max} < \overline J_R+\Delta = J_R^\star \le J_{\mathrm{reg}}(\pi^\star).
\label{eq:app_gap_chain}
\end{equation}
Thus every correctness-suboptimal $\pi$ has strictly smaller $J_{\mathrm{reg}}(\pi)$ than some correctness-optimal $\pi^\star$,
so no correctness-suboptimal policy can maximize $J_{\mathrm{reg}}$. Hence any maximizer satisfies~\eqref{eq:app_correct_opt}.
\end{proof}

\subsection{GRPO Signal Preservation via Correct-Only Global Diversity Reward}
\label{app:signal}

DSDR’s augmented reward (Eq.~\eqref{eq:augmented_reward}) is motivated by a concrete optimization issue in
group-relative methods: when verifier rewards become nearly constant within a group (e.g., many correct rollouts),
the within-group variance shrinks and the group-normalized advantages used in GRPO can degenerate.
This subsection formalizes two points aligned with Sec.~\ref{sec:dsdr}:
(i) verifier-only rewards yield informative groups only when the sampled group mixes successes and failures, and
(ii) DSDR’s correct-only diversity bonus creates controlled dispersion among correct solutions, ensuring non-degenerate
group-normalized advantages whenever diversity scores differ.

\begin{lemma}[Probability of a mixed verifier-reward group]
\label{lem:pmix}
Fix prompt $q$. Let $r_1,\dots,r_G$ be i.i.d.\ $\mathrm{Bernoulli}(p)$, where
\begin{equation}
p=\Pr(R(q,O)=1).
\label{eq:app_p}
\end{equation}
Let $\mathrm{Mix}=\{\exists i,j:\; r_i\neq r_j\}$ denote the mixed-group event. Then
\begin{equation}
\Pr(\mathrm{Mix})=P_{\mathrm{mix}}(p;G)=1-p^G-(1-p)^G.
\label{eq:app_pmix}
\end{equation}
\end{lemma}

\begin{proof}
The complement of $\mathrm{Mix}$ is the disjoint union of ``all ones'' and ``all zeros'',
with probabilities $p^G$ and $(1-p)^G$, respectively. Therefore
$\Pr(\mathrm{Mix})=1-p^G-(1-p)^G$.
\end{proof}

\begin{proposition}[Non-vanishing GRPO signal under correct-only diversity bonus]
\label{prop:nonvanishing}
Let $\tilde r_i$ be defined in~\eqref{eq:app_rtilde} and $\tilde A_i$ be the stabilized group-normalized advantages in~\eqref{eq:app_adv}.
If there exist $i\neq j$ such that $\tilde r_i\neq \tilde r_j$, then
\begin{equation}
\mathrm{Var}(\tilde r_1,\ldots,\tilde r_G)>0
\quad\text{and}\quad
\{\tilde A_i\}_{i=1}^G \text{ are not all zero.}
\label{eq:app_nonvanishing}
\end{equation}
In particular, in a solve-all group ($r_i=1$ for all $i$),
\begin{equation}
\tilde r_i = 1+\lambda_d \bar d_i.
\label{eq:app_solveall_rtilde}
\end{equation}
If $\lambda_d>0$ and $\{\bar d_i\}$ are not all identical, then group-relative advantages remain non-degenerate even though verifier rewards are constant.
\end{proposition}

\begin{proof}
If $\tilde r_i$ are not all identical, then letting $\mu_{\tilde r}$ be the group mean in~\eqref{eq:app_mean_std},
there exists an index $i$ such that $\tilde r_i\neq \mu_{\tilde r}$, hence
\begin{equation}
\mathrm{Var}(\tilde r_1,\ldots,\tilde r_G)
=
\frac{1}{G}\sum_{k=1}^G(\tilde r_k-\mu_{\tilde r})^2
>0.
\label{eq:app_var_pos}
\end{equation}
Because $\sigma_{\tilde r,\varepsilon}>0$ by construction, for that index $i$ we have
$\tilde A_i=(\tilde r_i-\mu_{\tilde r})/\sigma_{\tilde r,\varepsilon}\neq 0$,
so the advantages are not all zero. The solve-all case follows by substituting~\eqref{eq:app_solveall_rtilde}:
if $\lambda_d>0$ and $\bar d_i$ are not all equal, then $\tilde r_i$ are not all equal and the above applies.
\end{proof}

\paragraph{Remark (Connection to PPO/GRPO).}
In PPO-style clipped surrogate objectives~\citep{schulman2017proximal} and GRPO-style group-relative updates (e.g.,~\citep{guo2025deepseek}),
the reward-driven policy-improvement term is scaled by normalized advantages. If all advantages were zero, the reward-driven gradient would vanish
(leaving only KL regularization). Proposition~\ref{prop:nonvanishing} guarantees non-degenerate preference signal whenever diversity induces non-constant $\tilde r_i$.

\subsection{Optimality of Diversity-Softmax Global-to-Local Coupling }
\label{app:allocation}

DSDR couples global and local regularization by allocating the local entropy budget across correct trajectories
according to the diversity-softmax weights (Eq.~\eqref{eq:diversity_softmax}). This subsection shows that the
softmax allocation is not an ad-hoc heuristic: it is the unique solution of an entropy-regularized resource
allocation problem. This provides a principled interpretation of the temperature $\tau$ as a
concentration--exploration control knob.

\begin{proposition}[Softmax allocation optimality]
\label{prop:softmax_opt}
Assume $\mathcal P=\{i:\;r_i=1\}\neq\emptyset$. Consider allocating a local-entropy budget across correct rollouts with a distribution
$w$ over $\mathcal P$:
\begin{equation}
w_i\ge 0,\qquad \sum_{i\in\mathcal P} w_i = 1.
\label{eq:app_simplex}
\end{equation}
For $\tau>0$, the diversity-softmax weights
\begin{equation}
w_i^\star=\frac{\exp(\tau \bar d_i)}{\sum_{j\in\mathcal P}\exp(\tau \bar d_j)}
\label{eq:app_softmax}
\end{equation}
are the unique maximizer of the entropy-regularized allocation problem
\begin{equation}
\max_{w\in \Delta(\mathcal P)}
\Big(
\tau \sum_{i\in\mathcal P} w_i\,\bar d_i
+
\mathcal H(w)
\Big),
\qquad
\mathcal H(w)=-\sum_{i\in\mathcal P} w_i\log w_i,
\label{eq:app_alloc_obj}
\end{equation}
where $\Delta(\mathcal P)$ is the simplex in~\eqref{eq:app_simplex}.
\end{proposition}

\begin{proof}
This is the classical maximum-entropy / Gibbs distribution derivation~\citep{jaynes1957information}; we provide a convex-optimization proof via KKT conditions~\citep{boyd2004convex}.

\textbf{(Uniqueness)} The objective in~\eqref{eq:app_alloc_obj} is linear in $w$ plus entropy $\mathcal H(w)$, which is strictly concave over the simplex interior~\citep{cover1999elements,boyd2004convex}. Hence the objective is strictly concave and the maximizer is unique.

\textbf{(Stationarity)} Form the Lagrangian with multiplier $\alpha$ for $\sum_{i\in\mathcal P} w_i=1$:
\begin{equation}
\mathcal L(w,\alpha)
=
\tau\sum_{i\in\mathcal P} w_i \bar d_i
-\sum_{i\in\mathcal P} w_i\log w_i
+\alpha\Big(\sum_{i\in\mathcal P}w_i-1\Big).
\label{eq:app_lagr}
\end{equation}
At an interior optimum, stationarity gives
\begin{equation}
\frac{\partial \mathcal L}{\partial w_i}
=
\tau \bar d_i-(1+\log w_i)+\alpha=0
\quad\Rightarrow\quad
\log w_i=\tau \bar d_i+\alpha-1.
\label{eq:app_stationary}
\end{equation}
Exponentiating yields $w_i=C\exp(\tau\bar d_i)$ with $C=\exp(\alpha-1)$. Enforcing the simplex constraint implies
$C=\big(\sum_{j\in\mathcal P}\exp(\tau\bar d_j)\big)^{-1}$, yielding~\eqref{eq:app_softmax}.
By strict concavity, this stationary point is the unique global maximizer.
\end{proof}

\paragraph{Equivalent KL-Regularized Form.}
Let $u$ be the uniform distribution on $\mathcal P$. Using
$\mathcal H(w)=\log|\mathcal P| - D_{\mathrm{KL}}(w\|u)$, the problem~\eqref{eq:app_alloc_obj}
is equivalent (up to an additive constant) to
\[
\max_{w\in\Delta(\mathcal P)}
\left(
\tau\,\mathbb E_{i\sim w}[\bar d_i] - D_{\mathrm{KL}}(w\|u)
\right),
\]
highlighting the exploration--concentration tradeoff controlled by $\tau$.

\subsection{Proof of Theorem~\ref{thm:dsdr_tilted_pg}}
\label{app:proof_tilted_pg}

\begin{proof}
Fix $q$ and define
\[
f(q,o) \;=\; \exp\!\big(\tau \bar d(q,o)\big)\,\mathbbm{1}(R(q,o)=1)\;\;\ge 0.
\]
By definition,
\[
J_\tau(\theta;q)=\frac{1}{\tau}\log Z_\tau(\theta;q),
\qquad
Z_\tau(\theta;q)=\mathbb{E}_{o\sim \pi_\theta(\cdot\mid q)}[f(q,o)].
\]
Since $Z_\tau(\theta;q)>0$ by assumption, we can differentiate:
\begin{equation}
\label{eq:proof_step1}
\nabla_\theta J_\tau(\theta;q)
=
\frac{1}{\tau}\cdot \frac{1}{Z_\tau(\theta;q)}\;\nabla_\theta Z_\tau(\theta;q).
\end{equation}
Using the score-function (log-derivative) identity~\citep{williams1992simple,sutton1999policy},
\begin{equation}
\label{eq:proof_score}
\nabla_\theta \mathbb{E}_{o\sim \pi_\theta(\cdot\mid q)}[f(q,o)]
=
\mathbb{E}_{o\sim \pi_\theta(\cdot\mid q)}\!\left[
f(q,o)\,\nabla_\theta \log \pi_\theta(o\mid q)
\right],
\end{equation}
we obtain
\begin{equation}
\label{eq:proof_step2}
\nabla_\theta Z_\tau(\theta;q)
=
\mathbb{E}_{o\sim \pi_\theta(\cdot\mid q)}\!\left[
f(q,o)\,\nabla_\theta \log \pi_\theta(o\mid q)
\right].
\end{equation}
Substituting~\eqref{eq:proof_step2} into~\eqref{eq:proof_step1} yields
\begin{equation}
\label{eq:proof_step3}
\nabla_\theta J_\tau(\theta;q)
=
\mathbb{E}_{o\sim \pi_\theta(\cdot\mid q)}\!\left[
\frac{1}{\tau}\frac{f(q,o)}{Z_\tau(\theta;q)}\;\nabla_\theta \log \pi_\theta(o\mid q)
\right].
\end{equation}
Next, note that
\[
\mathbb{E}_{o\sim \pi_\theta(\cdot\mid q)}\!\left[\nabla_\theta \log \pi_\theta(o\mid q)\right]
=
\nabla_\theta \int \pi_\theta(o\mid q)\,do
=
\nabla_\theta 1
=
0,
\]
so subtracting the constant baseline $1/\tau$ does not change the expectation~\citep{sutton2000policy}:
\begin{equation}
\label{eq:proof_step4}
\nabla_\theta J_\tau(\theta;q)
=
\mathbb{E}_{o\sim \pi_\theta(\cdot\mid q)}\!\left[
\frac{1}{\tau}\Big(\frac{f(q,o)}{Z_\tau(\theta;q)}-1\Big)\;\nabla_\theta \log \pi_\theta(o\mid q)
\right].
\end{equation}
Expanding $f(q,o)$ and $Z_\tau(\theta;q)$ gives exactly~\eqref{eq:tilted_pg}--\eqref{eq:tilted_adv}.

Finally, for the Monte Carlo form, let $\{o_i\}_{i=1}^G$ be i.i.d.\ samples from $\pi_\theta(\cdot\mid q)$ and define
$f_i=f(q,o_i)=\exp(\tau \bar d_i)\mathbbm{1}(r_i=1)$.
A plug-in estimator for $Z_\tau(\theta;q)$ is $\hat Z=(1/G)\sum_{j=1}^G f_j$, and for the numerator
$\widehat{\nabla Z}=(1/G)\sum_{i=1}^G f_i \nabla_\theta \log \pi_\theta(o_i\mid q)$.
Thus the ratio form in~\eqref{eq:proof_step3} yields
\[
\widehat{\nabla_\theta J_\tau}(\theta;q)
=
\frac{1}{\tau}\cdot \frac{\widehat{\nabla Z}}{\hat Z}
=
\frac{1}{\tau}
\sum_{i=1}^G
\frac{f_i}{\sum_{j=1}^G f_j}\;
\nabla_\theta \log \pi_\theta(o_i\mid q),
\]
which corresponds to weights $\hat w_i=f_i/\sum_{j=1}^G f_j$, proving~\eqref{eq:tilted_softmax_mc}.
Because $f_i=0$ whenever $r_i=0$, restricting to the correct set $\mathcal P=\{i:r_i=1\}$ gives
$\hat w_i \propto \exp(\tau \bar d_i)$ on $\mathcal P$, matching Eq.~\eqref{eq:diversity_softmax}.
\end{proof}

\end{document}